\documentclass[letterpaper]{article} 
\usepackage{aaai25}  
\usepackage{times}  
\usepackage{helvet}  
\usepackage{courier}  
\usepackage[hyphens]{url}  
\usepackage{graphicx} 
\usepackage{natbib}  
\usepackage{caption} 
\usepackage{algorithm}
\usepackage{algorithmic}

\usepackage{newfloat}
\usepackage{listings}
\DeclareCaptionStyle{ruled}{labelfont=normalfont,labelsep=colon,strut=off} 
\floatstyle{ruled}
\newfloat{listing}{tb}{lst}{}
\floatname{listing}{Listing}
\usepackage{hyperref} 

\usepackage{booktabs}
\usepackage{amsmath}
\usepackage{amssymb}
\usepackage{mathtools}
\usepackage{amsthm}
\usepackage{enumitem}
\usepackage{enumitem}
\usepackage{subfigure}
\theoremstyle{plain}
\newtheorem{theorem}{Theorem}
\newtheorem{proposition}[theorem]{Proposition}
\newtheorem{lemma}[theorem]{Lemma}

\theoremstyle{definition}

\theoremstyle{remark}

\usepackage[textsize=tiny]{todonotes}
\usepackage[capitalize,noabbrev]{cleveref}

\title{Rapid Learning in Constrained Minimax Games with Negative Momentum}
\author{
    Zijian Fang\textsuperscript{\rm 1}\equalcontrib,
    Zongkai Liu\textsuperscript{\rm 1, \rm 3}\equalcontrib,
    Chao Yu\textsuperscript{\rm 1, \rm 2}\thanks{Corresponding author.},
    Chaohao Hu\textsuperscript{\rm 1}
}
\affiliations {
    \textsuperscript{\rm 1}School of Computer Science and Engineering, Sun Yat-sen University, Guangzhou, China\\
    \textsuperscript{\rm 2}Pengcheng Laboratory, Shenzhen, China\\
    \textsuperscript{\rm 3}Shanghai Innovation Institute, Shanghai, China\\
    \{fangzj, liuzk\}@mail2.sysu.edu.cn, {yuchao3}@mail.sysu.edu.cn,
    {huchh9}@mail3.sysu.edu.cn

}

\usepackage{bibentry}

\begin{document}

\maketitle

\begin{abstract}
In this paper, we delve into the utilization of the negative momentum technique in constrained minimax games. From an intuitive mechanical standpoint, we introduce a novel framework for momentum buffer updating, which extends the findings of negative momentum from the unconstrained setting to the constrained setting and 
provides a universal enhancement to the classic game-solver algorithms. Additionally, we provide theoretical guarantee of convergence for our momentum-augmented algorithms with entropy regularizer. We then extend these algorithms to their extensive-form counterparts. Experimental results on both Normal Form Games (NFGs) and Extensive Form Games (EFGs) demonstrate that our momentum techniques can significantly improve algorithm performance, surpassing both their original versions and the SOTA baselines by a large margin.
\begin{links}
    \link{Code}{https://github.com/kkkaiaiai/NM-Method}
\end{links}
\end{abstract}

\section{Introduction}

In recent years, a broad spectrum of applications in machine learning and robust optimization have been cast as a minimax optimization problem in the form of $\min_{\boldsymbol x \in \mathcal X} \max_{\boldsymbol y \in \mathcal Y} f(\boldsymbol x, \boldsymbol y)$. Examples formulated under this framework include generative adversarial networks (GANs)~\cite{DBLP:journals/cacm/GoodfellowPMXWO20}, adversarial training~\shortcite{DBLP:conf/iclr/SinhaND18}, fair statistical inference~\cite{DBLP:conf/icml/MadrasCPZ18}, market equilibrium~\cite{kroer2019computing}, primal-dual reinforcement learning~\cite{du2017stochastic} and numerous others. This optimization problem can be conceptualized as a zero-sum game involving two players: the first player minimizes $f(\boldsymbol x, \boldsymbol y)$ by tuning $\boldsymbol x$, while the other player maximizes $f(\boldsymbol x, \boldsymbol y)$ by tuning $\boldsymbol y$.

With the inextricably intertwined advancement of online learning and game theory, several algorithms have become foundational solvers for minimax games, including the Online Mirror Descent (OMD)~\cite{warmuth1997continuous}, Follow-The-Regularized-Leader (FTRL)~\cite{DBLP:conf/colt/AbernethyHR08}, and particularly Regret Matching (RM)~\cite{hart2000simple}, which is an algorithm more closely aligned with game theory and becomes the building block of solving imperfect-information games~\cite{moravvcik2017deepstack, brown2018superhuman}. 
However, these algorithms are known to exhibit rotation behaviour and fail to converge pointwise even in simple bilinear cases~\cite{DBLP:conf/nips/Vlatakis-Gkaragkounis19a}. A corpus of studies aim to address this divergence issue through plain modifications of standard algorithms, 
with a specific focus on achieving an enhanced convergence rate and/or securing last-iterate convergence guarantees~\cite{DBLP:conf/nips/GolowichPD20}.

Among these techniques, regularization and optimistic gradient are two widely used methods, with Magnet Mirror Descent (MMD)~~\cite{DBLP:conf/iclr/SokotaDKLLMBK23} and Optimistic Gradient Descent Ascent (OGDA)~~\cite{DBLP:conf/iclr/MertikopoulosLZ19} as their respective representative algorithms.
MMD provides a linear convergence rate to the regularized equilibrium by utilizing the influences of regularization on last-iterate convergence.
OGDA, on the other hand, introduces an optimistic gradient estimate to guide the convergence process more effectively by predicting future gradients.
In addition, Negative Momentum (NM)  is introduced as an enhancement technique in recent research \cite{DBLP:conf/aistats/GidelHPPHLM19}, achieving a linear convergence rate comparable to Extragradient \cite{korpelevich1976extragradient} and OGDA in unconstrained bilinear games. Subsequent works have further delved into the convergence properties of NM~\cite{zhang2021suboptimality, lorraine2022complex}. 
However, 
recent analysis of NM  predominantly focus on the impact of its integration with GDA and the learning dynamics over the unconstrained setting, leaving a gap in discussions concerning  its interaction with game-solver algorithms and its performance in constrained settings. 
Hence, this paper places particular emphasis on the following two questions:

\begin{itemize}
  \item Can NM be extended from the unconstrained setting to the constrained  setting?
  \item Can NM provide a significant  empirical improvement over existing methods like regularization and optimistic?
  
\end{itemize}

To provide affirmative answers to these questions, we make the following  contributions in this paper:

\begin{itemize}
  \item We introduce a negative momentum updating framework tailored for the constrained setting, coupled with an intuitive paradigm for updating the momentum buffer, which can be seamlessly integrated with classic algorithms. Furthermore, by using the dilated distance generated function~\cite{hoda2010smoothing} and regret decomposition framework~\cite{farina2019online}, we propose momentum-augmented versions of their extensive-form counterparts.
  \item   We theoretically prove  that our momentum-augmented variant with negative entropy regularizer achieves an exponential convergence rate to an approximate equilibrium with an infinitely large buffer or converges to the set of Nash equilibria with a sufficiently large buffer. 
  \item We conduct comprehensive experiments over randomly generated NFGs and four standard EFGs, including Kuhn Poker, Leduc Poker, Goofspiel and Liar’s dice. 
  The experimental results demonstrate that the momentum-augmented algorithms exhibit significant improvements over both their original versions and other existing strong last-iterate convergent baselines.
  It is noteworthy that our proposed algorithms $\text{MoRM}^{+}$($\text{MoCFR}^{+}$)  consistently obtain $10^{9}$ times lower exploitability than $\text{RM}^{+}$($\text{CFR}^{+}$), and outperform another SOTA variant $\text{PCFR}^{+}$. To our knowledge, this marks the first instance where an algorithm surpasses $\text{CFR}^{+}$ across various types of games.
\end{itemize}

\section{Related Work}

The related work is organized to encompass  existing general techniques for facilitating minimax training and addressing the convergence problem over both the unconstrained and constrained setting.

\subsubsection{Timescale separation.} Timescale separation involves solving the inside maximization problem initially to get an approximation of $\boldsymbol{y}^{*}$ and compute the gradient of $\boldsymbol{x}$ as if $\boldsymbol{y}^{*}$ is fixed, serving as a potential good descent direction. 
In training GANs, \citet{DBLP:conf/nips/HeuselRUNH17} utilize a larger learning rate for the discriminator to ensure convergence to a local Nash Equilibrium (NE).
\citet{fiez2021local} explore more
general non-convex non-concave zero-sum games and elucidate the local convergence to strict local
minimax equilibrium with finite timescale separation. 
The two-timescale update rule resembles a softened ``learning vs. best response''
scheme~\cite{DBLP:conf/nips/DaskalakisFG20}, which has also been investigated in the literature of the constrained setting like game solving, guaranteeing the convergence to the NE~\cite{ DBLP:conf/ijcai/LockhartLPLMTT19}. Nevertheless, the faster-updating player requires training a costly best response oracle at each iteration, and asymmetric updates typically lead to less desirable unilateral convergence.

\subsubsection{Predictive updates.} Predictive updates come from the intuitions that players could utilize heuristics to predict each other's next move~\cite{Foerster2017LearningWO, chavdarova2021tamingganslookaheadminmax}. Algorithmically, these methods can be considered variations of the optimistic/extra-gradient methods, where the gradient dynamics are modified by incorporating approximate second-order information~\cite{DBLP:conf/nips/SchaferA19}. 
Previous studies have investigated the last-iterate convergence in the unconstrained setting such as training GANs~\cite{DBLP:conf/aistats/LiangS19}.
In cases where a unique NE is assumed, 
\citet{DBLP:conf/innovations/DaskalakisP19} and \citet{DBLP:conf/iclr/WeiLZL21} have extended the scope of research for Optimistic Multiplicative Weight Update (OMWU) in NFGs. 
In the context of EFGs, \citet{DBLP:conf/nips/FarinaKS19} empirically demonstrate the last-iterate convergence of OMWU, while \citet{DBLP:conf/nips/LeeKL21} subsequently establish theoretical proofs with the uniqueness assumption of NE. 
Recent work by \cite{farina2023regretmatchinginstabilityfast, cai2023lastiterateconvergencepropertiesregretmatching} has extended the analysis to include tighter ergodic convergence rate and last-iterate convergence.
While the predictive updates algorithms are normally accompanied by theoretical properties, their practical implementation often necessitates the computation of multiple strategies at each iteration. Furthermore, these algorithms may not consistently yield a significant acceleration, especially in games with more intricate structures~\cite{DBLP:conf/aaai/FarinaKS21} or with a larger scale~\cite{DBLP:conf/nips/LeeKL21}.

\subsubsection{Regularization.}
Regularization has emerged as a pivotal tool for accelerating convergence. \citet{DBLP:conf/icml/PerolatMLOROBAB21} conduct a comprehensive analysis of the impact of entropy regularization on continuous-time dynamics, and propose a reward transformation method to achieve linear convergence in EFGs using counterfactual values. However, their theoretical findings cannot be inherently extended to desired discrete-time results.  
\citet{DBLP:journals/corr/abs-2208-09855} propose a variant of MWU by incorporating an additional regularization term serving as the mutation dynamic, while \citet{liu2023powerregularizationsolvingextensiveform} achieve improved convergence results by regularizing the payoff functions of the games.
Magnet Mirror Descent (MMD)~\cite{DBLP:conf/iclr/SokotaDKLLMBK23} investigate the influences of general-case regularization on last-iterate convergence and provide a linear convergence rate to the regularized equilibrium. In a parallel study, \citet{abe2024adaptivelyperturbedmirrordescent} yield comparable results but aim for an exact Nash equilibrium while imposing a more stringent constraint on the learning rate.

\subsubsection{Other techniques.}
Other methods modify algorithms with ad-hoc adjustments to game dynamics. Consensus optimization (CO)~\cite{DBLP:conf/nips/MeschederNG17} and gradient penalty~\cite{DBLP:conf/nips/GulrajaniAADC17} improve convergence by minimizing the players' gradient magnitude. 
\citet{DBLP:conf/icml/BalduzziRMFTG18} improve convergence by disentangling convergent potential components from rotation Hamiltonian components of vector field. 
However, these approaches require estimating coupled gradient and Hessian information, which is computationally expensive, prone to high variance, and less practical even in centralized settings.

\section{Preliminaries}
\subsection{Problem Formulation}
In this paper, we consider the problem of solving a constrained bilinear \textit{saddle-point problem} (SPP):
\begin{equation}
\min_{\boldsymbol x \in \mathcal X} \max_{\boldsymbol y \in \mathcal Y}\boldsymbol b^{\top}\boldsymbol x+\boldsymbol x^{\top}\boldsymbol G\boldsymbol y+\boldsymbol c^{\top}\boldsymbol y,
\label{bilinear form}
\end{equation}
where $G\in[-1,+1]^{M\times N}$ is a known loss function matrix, and $\mathcal X \subset \mathbb R^{M}$, $\mathcal Y \subset \mathbb R^{N}$ are the convex and compact decision sets (i.e. strategies) for min/max players. The linear terms are not essential in our analysis and thus we take $\boldsymbol{b}=\boldsymbol{c}=0$ throughout the paper\footnote{If they are not zero, one can translate $\boldsymbol x$ and $\boldsymbol y$ to cancel the linear terms by linear transform, see e.g. \cite{DBLP:conf/aistats/GidelHPPHLM19}.}. 
This problem formulation captures several game-theoretical applications such as finding the NE in norm-form/extensive-form zero-sum games when $\mathcal X$ and $\mathcal Y$ represent simplex $\Delta$ and treeplex $\Delta_{T}$, respectively. We denote the simplex/treeplex of dimension $d-1$ as $\Delta^{d}/\Delta_{T}^{d}$. By the celebrated minimax theorem~\cite{v1928theorie}, we have $\min_{\boldsymbol{x} \in \mathcal{X}} \max _{\boldsymbol{y} \in \mathcal{Y}} \boldsymbol{x}^{\top} \boldsymbol{G} \boldsymbol{y} =\max_{\boldsymbol{y} \in \mathcal{Y}} \min _{\boldsymbol{x} \in \mathcal{X}} \boldsymbol{x}^{\top} \boldsymbol{G} \boldsymbol{y}$. The set of Nash equilibria is defined as $\mathcal Z^{*}=\mathcal{X}^{*} \times \mathcal{Y}^{*}$ where $\mathcal{X}^{*}=\operatorname{argmin}_{\boldsymbol{x} \in \mathcal{X}} \max _{\boldsymbol{y} \in \mathcal{Y}} \boldsymbol{x}^{\top} \boldsymbol{G} \boldsymbol{y}$ and $ \mathcal{Y}^{*}=\operatorname{argmax}_{\boldsymbol{y} \in \mathcal{Y}} \min _{\boldsymbol{x} \in \mathcal{X}} \boldsymbol{x}^{\top} \boldsymbol{G} \boldsymbol{y}$, which is always convex for two-player zero-sum games. The duality gap (i.e. exploitability) of a pair of feasible strategies $\boldsymbol{z}=(\boldsymbol{x}, \boldsymbol{y}) \in \mathcal{Z}=\mathcal{X}\times\mathcal{Y}$ is defined as:
\begin{equation}
\begin{aligned}
\textit{DualityGap}(\boldsymbol{x}, \boldsymbol{y}) = \max_{\boldsymbol y^{\prime} \in \mathcal Y}\boldsymbol x^{\top}\boldsymbol G\boldsymbol y^{\prime} - \min_{\boldsymbol x^{\prime} \in \mathcal X}\boldsymbol x^{\prime^\top}\boldsymbol G\boldsymbol y.
\end{aligned}
\end{equation}
Note that $\textit{DualityGap}(\boldsymbol{x}, \boldsymbol{y}) \geq 0$ holds and $\textit{DualityGap}(\boldsymbol{x}, \boldsymbol{y}) \leq \epsilon$ implies that the strategy profile $\boldsymbol{z} \in \mathcal{Z}$ is an $\epsilon$-Nash equilibrium of the bilinear game.

For notation convenience, we let $P=M+N$ and denote the loss vector of the bilinear form in Equation (\ref{bilinear form}) as $F(\boldsymbol z_{t}) = (F(\boldsymbol x_{t}), F(\boldsymbol y_{t})) = (\boldsymbol G\boldsymbol y_{t}, -\boldsymbol G^{\top} \boldsymbol x_{t}) = (\boldsymbol f_{t},-\boldsymbol g_{t})$ for any $\boldsymbol z_{t}=(\boldsymbol x_{t},\boldsymbol y_{t}) \in \mathcal{Z} = \mathcal{X} \times \mathcal{Y} \subset \mathbb{R}^{P}$, where $\boldsymbol f_{t}$ and $\boldsymbol g_{t}$ represent the gradients of the current strategy profile. We assume $\|F(\boldsymbol z)\|_{\infty} \leq 1$ for all $\boldsymbol z \in \mathcal{Z}$, which can be always satisfied by normalizing the entries of $G$.

One way to solve bilinear SPPs is by viewing SPP as a repeated game between two players: at iteration $t$, players choose $\boldsymbol{z}_{t} \in \mathcal{Z}$ and then observe their loss $l_{t}^{\mathcal{Z}}\left(\boldsymbol{z}_{t}\right) = \left(\boldsymbol{x}_{t}^{\top} \boldsymbol{G} \boldsymbol{y}_{t}, -\boldsymbol{x}_{t}^{\top} \boldsymbol{G} \boldsymbol{y}_{t}\right) = \left\langle\boldsymbol{z}_{t}, F(\boldsymbol{z}_{t})\right\rangle$. The goal of each player is to minimize their regrets $R_{T,\boldsymbol x}$, $R_{T,\boldsymbol y}$ across $T$ iterations:
\begin{equation}
\begin{aligned}
R_{T, \boldsymbol{x}}&=\sum_{t=1}^{T}\left\langle\boldsymbol{f}_{t}, \boldsymbol{x}_{t}\right\rangle-\min _{\boldsymbol{x} \in \mathcal{X}} \sum_{t=1}^{T}\left\langle\boldsymbol{f}_{t}, \boldsymbol{x}\right\rangle, 
\\ 
\quad R_{T, \boldsymbol{y}}&=\max _{\boldsymbol{y} \in \mathcal{Y}} \sum_{t=1}^{T}\left\langle\boldsymbol{g}_{t}, \boldsymbol{y}\right\rangle-\sum_{t=1}^{T}\left\langle\boldsymbol{g}_{t}, \boldsymbol{y}_{t}\right\rangle ,
\end{aligned}
\label{regret definition}
\end{equation}
which measure the difference between the loss accumulated by the sequence of $(\boldsymbol{z}_{1},...\boldsymbol{z}_{T})$ and the loss that would have been accumulated by employing the best time-independent strategies $\boldsymbol z$ in hindsight. An algorithm is called  \textit{regret minimizer} if the regret grows sublinearly in $T$. We denote $R_{T, \hat{\boldsymbol{x}}}$ as the regret of arbitrary $\hat{\boldsymbol{x}} \in \mathcal{X}$.

It is well known that if $\boldsymbol{z}_{t}$ follows the trajectory of a \textit{regret minimizer} learning algorithm, $\frac{1}{t}\sum_{\tau \leq t}\boldsymbol x_{\tau}^{\top}\boldsymbol G\boldsymbol y_{\tau}$ converges to the optimal value of the bilinear SPP (\ref{bilinear form}) as $t \rightarrow \infty$, and the average strategies $\frac{1}{t}\sum_{\tau \leq t}\boldsymbol z_{\tau}$ converges to the optimal solution to the SPP (i.e. NE) if the solution is unique. 

\subsection{Online Learning and Regret Matching}
\label{No-regret Algorithms}

\subsubsection{Online Linear Optimization Oracles.} The online optimization oracles all follow a reminiscent procedure that within the operation loop of observing the loss vector $F(\boldsymbol z_{t})$ and updating the next strategies $\boldsymbol{z}_{t+1}$. For solving bilinear SPPs over constrained sets, Online Mirror Descent (OMD)~\cite{warmuth1997continuous}
and Follow-The-Regularized-Leader (FTRL)~\cite{DBLP:conf/colt/AbernethyHR08} stand out as two most classical online linear optimization algorithms.
With arbitrary $\boldsymbol{z}_{0} \in \mathcal{Z}$, the OMD algorithm proceeds iterations following the rule:
\begin{equation}
\boldsymbol{z}_{t+1} =\underset{\boldsymbol{z} \in \mathcal{Z}}{\operatorname{argmin}}\left\{\eta\left\langle\boldsymbol{z}, F(\boldsymbol{z}_{t})\right\rangle+D_{\psi}\left(\boldsymbol{z}, \boldsymbol{z}_{t}\right)\right\}, 
\label{OMD}
\end{equation}
where $\psi$ is a convex function called \textit{regularizer}, $\psi(\boldsymbol{z}) = \psi(\boldsymbol{x}) + \psi(\boldsymbol{y})$ and $D_{\psi}(p,q) = \psi(p) - \psi(q) - \langle{\nabla\psi(q), p - q}\rangle$ is the \textit{Bregman divergence}. The FTRL produces iterations following the rule:
\begin{equation}
\boldsymbol{z}_{t+1} =\underset{\boldsymbol{z} \in \mathcal{Z}}{\operatorname{argmin}}\left\{\sum^{t}_{k=1}\eta\left\langle\boldsymbol{z}, F(\boldsymbol{z}_{k})\right\rangle+\psi (\boldsymbol z)\right\}.
\label{FTRL}
\end{equation}
 FTRL can be equivalent to OMD with linearized loss~\cite{orabona2023modernintroductiononlinelearning}. OMD/FTRL instantiate Gradient Descent Ascent (GDA) and Multiplicative Weights
Update (MWU) when regularizer $\psi(\boldsymbol{z})$ is specified as the negative entropy $\psi(\boldsymbol{z}) = \boldsymbol z\log \boldsymbol z$ or the $L_{2}$-norm $\psi(\boldsymbol{z}) = \frac{1}{2}\|\boldsymbol{z}\|^{2}$, respectively.
\subsubsection{Regret Matching.} Regret Matching (RM)~\cite{hart2000simple} stands as one of the preeminent \textit{regret minimizer}  learning algorithms, extensively utilized in the domain of game-solving applications. An instantaneous regret vector is defined as $\boldsymbol r(\boldsymbol x_{t}) = \left\langle\boldsymbol x_{t}, F(\boldsymbol x_{t})\right\rangle \cdot \boldsymbol1_{L}-F(\boldsymbol x_{t})$,\footnote{$L$ can be either $M$ or $N$, that is we overload the notation $r$ so its domain depends on the input.} which measures the change in regret incurred at iteration $t$ relative to each dimension of the decision sets. RM keeps an accumulative regret $\boldsymbol R^{x}_{t}$ and updates strategy $\boldsymbol x_{t+1}$ by normalizing the thresholded accumulative regret:
\begin{equation}
\boldsymbol R^{x}_{t+1} = \sum_{\tau=0}^{t} \boldsymbol r(\boldsymbol x_{\tau}),\quad 
\boldsymbol x_{t+1} = \left[\boldsymbol R^{x}_{t+1}\right]^{+}\big/\|\left[\boldsymbol R^{x}_{t+1}\right]^{+}\|_{1},
\end{equation}
where $\left[\cdot\right]^{+}$ denotes thresholding at zero. $\text{Regret Matching}^{+}$ ($\text{RM}^{+}$)~\cite{tammelin2014solvinglargeimperfectinformation} is a variant of RM that further thresholds the accumulative regret at zero at every iteration: $\boldsymbol R^{x}_{t+1} = \left[\boldsymbol R^{x}_{t} + \boldsymbol r(\boldsymbol x_{t})\right]^{+}$. \citet{burch2019revisiting} show that combining $\text{RM}^{+}$ with the alternation trick, wherein the strategies of two players are updated asynchronously, yields faster empirical performance and is proven to exhibit strict improvement for game solving~\cite{grand2023solving}. 
In contrast to FTRL/OMD, the update rules of RM stand out for being parameter-free and exclusively involving closed-form operations, specifically, thresholding and normalizing.

\section{Methods}
\subsection{The Negative Momentum Mechanism}
We recall that in the unconstrained case, the iterations of the Gradient Descent Ascent with augmented (Polyak) momentum~\cite{polyak1964some} term (GDAm) proceed as follows:
\begin{equation}
\begin{aligned}
\boldsymbol z_{t+1} = \boldsymbol z_{t} - \eta F(\boldsymbol z_{t}) + \beta(\boldsymbol z_{t} - \boldsymbol z_{t-1}),
\label{GDAm}
\end{aligned}
\end{equation}
where $\eta$ is a positive step size. Alternatively, with $\boldsymbol \mu_{0}=0$, the procedure can be written in an equivalent form with momentum buffer $\boldsymbol \mu_{t}=(\boldsymbol z_{t}- \boldsymbol z_{t-1})/\eta$:
\begin{equation}
\begin{aligned}
\boldsymbol \mu_{t}=\beta \boldsymbol \mu_{t-1} -  F(\boldsymbol z_{t}), \quad 
\boldsymbol z_{t}=\boldsymbol z_{t-1} + \eta \boldsymbol \mu_{t}. 
\end{aligned}
\label{momentumbuffer}
\end{equation}

 The convergence process of GDAm can be illustrated by the Newton's $2^{nd}$ law $m\ddot{X} = \boldsymbol F$ of a particle of mass $m$. Without loss of generality, we henceforth assume the mass of our object to be unity. By viewing the $\boldsymbol F_\text{curl} = -F(\boldsymbol{z})$ as the curl force of the 2-dimensional $X-Y$ surface~\cite{berry2016curl}, the discretization of the continuous-time dynamic $ \ddot{Z} = -F(Z)$ with the discretization step size $\delta=\sqrt{\eta}$ can be written as follows:

\begin{equation}
\begin{aligned}
\boldsymbol z_{t+1} = 2\boldsymbol z_{t}  - \boldsymbol z_{t-1} - \eta F(\boldsymbol z_{t}) ,
\label{GDAm，1}
\end{aligned}
\end{equation}
which is corresponding to setting $\beta = 1$ in Equation (\ref{GDAm}).
\begin{figure}[ht]
  \centering
  \subfigure[]{
    \includegraphics[width=0.21\textwidth]{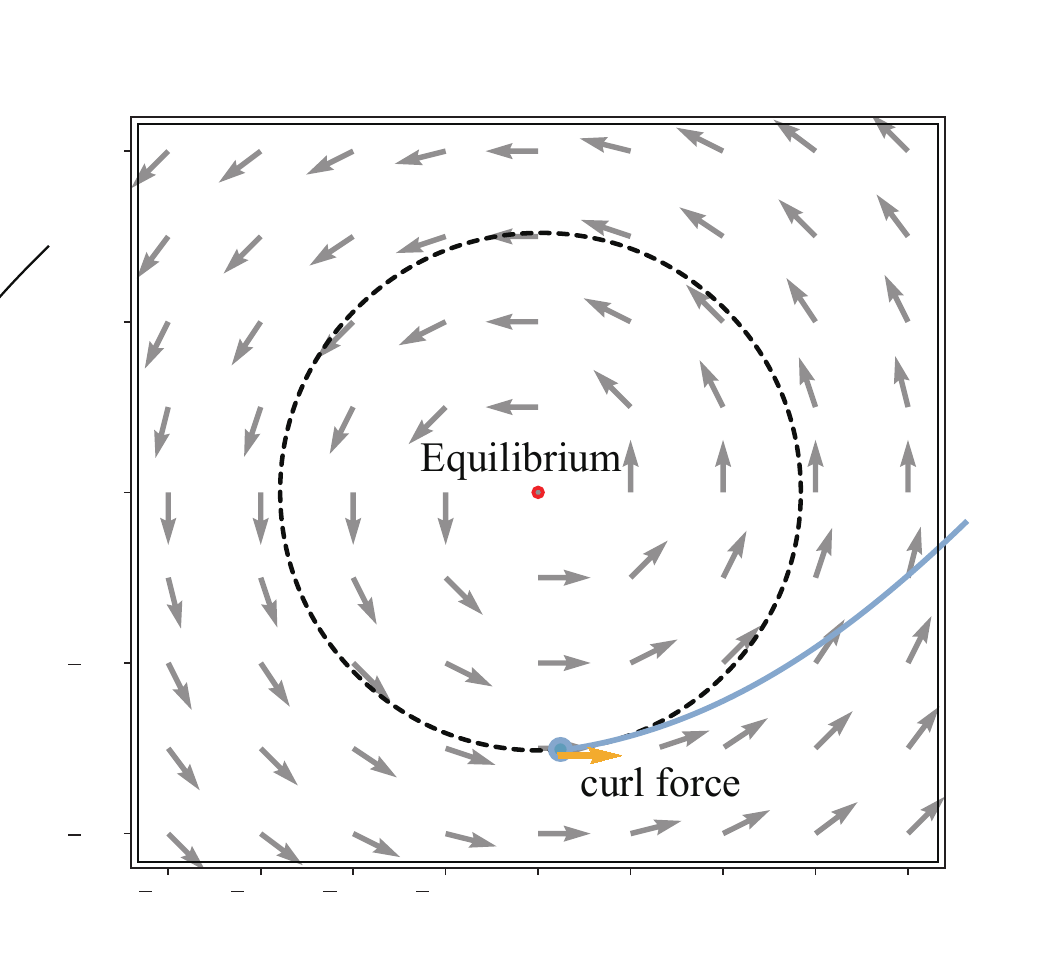}
    \label{curl}
    
  }
  \hfill
  \subfigure[]{
    \includegraphics[width=0.21\textwidth]{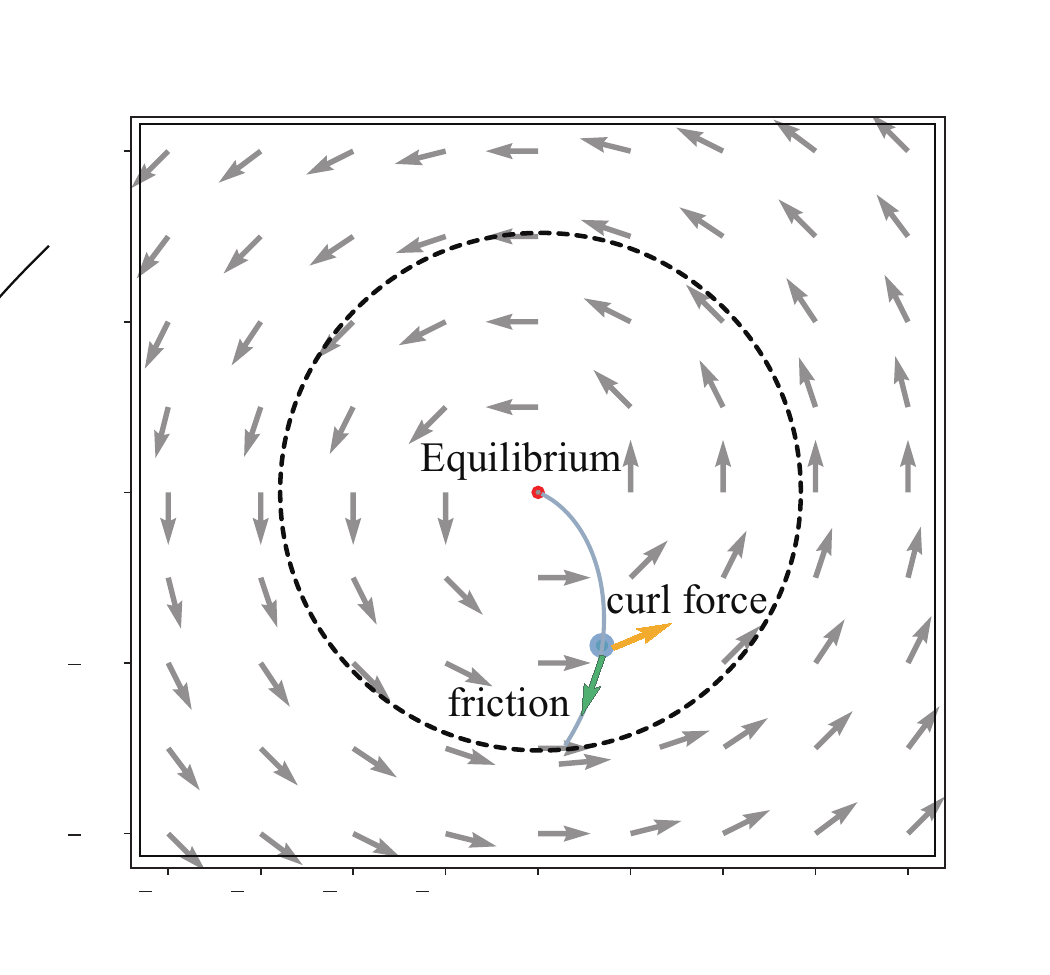}
    \label{friction}
  }
  \caption{The mechanic dynamic of a particle with different force. In (a), the particle situated within a curl force diverges from the equilibrium. In (b), the augmented friction results in a reduction of the particle's velocity, thereby dampening oscillations and facilitating eventual convergence.}
  \label{mechanic dyanmic}
\end{figure}

Nevertheless, as shown in Figure \ref{curl}, the mode of GDAm with non-negative parameters is divergent in the min-max objective since the curl force increases the particle's speed over time and thus prevents convergence. 
Therefore, to decrease the velocity of particle for convergence, a straightforward approach is augmenting the dynamic with an extra linear friction $\boldsymbol{F}_\text{fric} = -\mu\dot{Z}$ as follows:
\begin{align}\label{eq: continuous negative}
    \ddot{Z} = -F(Z) - \mu\dot{Z},
\end{align}
where $\mu$ can be considered as the coefficient of friction.  
Figure \ref{friction} illustrates that the friction in Equation (\ref{eq: continuous negative}) effectively dampens oscillation and leads to particle's convergence. Furthermore, through a combination of implicit and explicit update steps as \citet{DBLP:conf/nips/ShiDSJ19}, Proposition \ref{prop: unconstrained NM} illustrates that the discretization of Equation (\ref{eq: continuous negative}) corresponds to the GDAm with negative momentum.  

\begin{proposition}\label{prop: unconstrained NM}
    With the discretization step-size $\delta=\sqrt{\eta}$ and a sufficiently large $\mu > \frac{1}{\delta}$, Equation (\ref{eq: continuous negative}) can be discretized in the form of GDAm with negative momentum:
    \begin{align}
        \boldsymbol z_{t+1} = \boldsymbol z_{t} - \eta F(\boldsymbol z_{t}) +  \beta(\boldsymbol z_{t} - \boldsymbol z_{t-1}),
    \end{align}
    where $\beta=1 - \mu \delta < 0$.
\end{proposition}

\subsection{Generalization to The Constrained Setting}
Motivated by the effect of negative momentum terms, a key objective of our work is to extend these findings to the constrained setting. Inspired by Eq. \ref{momentumbuffer}, we introduce a momentum buffer defined by $\boldsymbol \mu_{t}=\beta \boldsymbol \mu_{t-1} - F(\boldsymbol{z}_{t})$ with $\boldsymbol{\mu}_{0}=0$. This allows us to replace $F(\boldsymbol{z}_{t})$ in Eq. \ref{OMD} with the negative momentum term $\boldsymbol{\mu}_{t}$ and propose the Mirror Descent with Momentum (MoMD) updating rules:
\begin{equation}
\boldsymbol{z}_{t+1} =\underset{\boldsymbol{z} \in \mathcal{Z}}{\operatorname{argmin}}\left\{\eta\left\langle\boldsymbol{z}, \boldsymbol -\boldsymbol \mu_{t}\right\rangle+D_{\psi}\left(\boldsymbol{z}, \boldsymbol{z}_{t}\right)\right\}. \\
\label{MMDm}
\end{equation}
Following the equivalence between OMD and FTRL, the updating rules of the FTRL with Momentum (MoFTRL) can be written as:  
\begin{equation}
\begin{aligned}
\boldsymbol L_{t}= \boldsymbol L_{t-1}  - \boldsymbol \mu_{t} &= \boldsymbol L_{t-1}  + F(\boldsymbol{z}_{t}) -\beta (\boldsymbol L_{t-2} - \boldsymbol L_{t-1}),  \\
\boldsymbol{z}_{t+1} &=\underset{\boldsymbol{z} \in \mathcal{Z}}{\operatorname{argmin}}\left\{\eta\left\langle\boldsymbol{z}, \boldsymbol L_{t} \right\rangle+\psi (\boldsymbol z)\right\}.
\end{aligned}
\label{FTRLm}
\end{equation}

\begin{figure*}[ht]
\begin{center}
\centerline{\includegraphics[width=2\columnwidth]{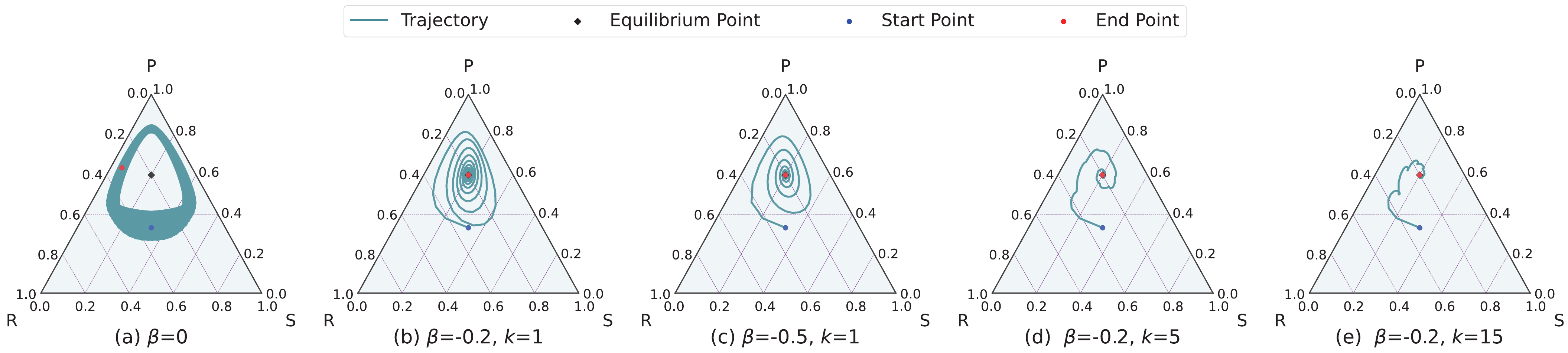}}
\caption{The trajectory plots depict the MoMWU algorithm under varying negative momentum coefficients $\beta$ and intervals $k$ in Bias RPS. The initial strategy is set to $(\boldsymbol{x}, \boldsymbol{y}) = \left(\frac{1}{3}, \frac{1}{3}, \frac{1}{3}\right)$. The equilibrium strategy is denoted by a black point, whereas the blue and red points symbolize the starting and ending points of the trajectory. The learning rate $\eta$ is fixed at 2 in this context.}
\label{trajectory}
\end{center}
\end{figure*}

It is noteworthy that the negative momentum term is kept in the dual space rather than the prime space. Taking an alternative perspective, negative momentum involves leveraging the buffer to store historical gradients and alternates between adding and subtracting the gradient at each iteration. 

To intensify the focus on (i.e., allocate greater weights to) gradients from recent iterations, inspired by prior research on momentum buffer variants~\cite{lorraine2022complex}
, we additionally propose to cache a profile $\mathcal{L}$ saving buffer snapshots within a fixed time duration $k$ and aggregate them for the final updates. 
That is, employing a degree of notation flexibility, we transform the formula for updating the momentum from $\boldsymbol{u}_t=\beta \boldsymbol{u}_{t-1}-F(\boldsymbol{z}_t)$ to $\boldsymbol{u}_t=\beta\sum_{\boldsymbol{u}\in\mathcal{L}}\boldsymbol{u} - F(\boldsymbol{z}_t)$ in Equation (\ref{MMDm}) or (\ref{FTRLm}), 
and the profile is reset when the length $L$ of $\mathcal{L}$ reaches specified $k$. This updating paradigm, denoted as  Restarting Aggregated Momentum (RAM), is illustrated in Algorithm \ref{Restarting Aggregated Momentum}.

\begin{algorithm}[tb]
   \caption{Restarting Aggregated Momentum (RAM)}
   \label{Restarting Aggregated Momentum}
\begin{algorithmic}[1]
   \STATE {\bfseries Input:} $\mathcal{L}$, $\mu$, $\beta$, Restarting interval $k$
   \STATE $\mathcal L \leftarrow \emptyset, \boldsymbol \mu_{0} = 0,\boldsymbol z_{0} \in \mathcal{Z}$ 
   \FOR{$t=0$ {\bfseries to} $T$}
   \STATE $\boldsymbol \mu_{t}=\beta \boldsymbol \sum_{\boldsymbol\mu \in \mathcal L}\boldsymbol{\mu} - F(\boldsymbol{z}_{t})$
   \STATE $\boldsymbol z_{t+1}=\textit{Oracle}(\boldsymbol z_{t}, \boldsymbol\mu_{t})$ \quad \text{following (\ref{MMDm}) or (\ref{FTRLm})}
   \IF{$L = k$}
    \STATE Reset the snapshot profile $\mathcal L \rightarrow \emptyset$ 
   \ENDIF
   \STATE Save current momentum $\left\{\boldsymbol\mu_{t}\right\} \cup \mathcal L \rightarrow \mathcal L$
   \ENDFOR
\STATE {\bfseries Output:} $\boldsymbol z_{t+1}$
\end{algorithmic}
\end{algorithm}

Taking MoFTRL as an example, if we substitute
$\boldsymbol \mu_{t}=\boldsymbol L_{t-1}- \boldsymbol L_{t}$ into the aggregation of snapshots profile and cancel out the interleave terms, the updating rules of the cumulative past loss in Equation (\ref{FTRLm}) become:
\begin{equation}
\begin{aligned}
\boldsymbol L_{t} = &\boldsymbol L_{t-1}  + F(\boldsymbol{z}_{t}) - \beta (\boldsymbol L_{att}- \boldsymbol L_{t-1}),  \\
&\boldsymbol L_{att} \leftarrow \boldsymbol L_{t-1} \quad\text{if $t$} \ \% \ k = 0.
\end{aligned}
\end{equation}
Intuitively, the momentum updating paradigm within Algorithm \ref{Restarting Aggregated Momentum} resembles mounting a spring between a periodically updating attachment point $ L_{att}$ and the current point $ L_{t}$, transforming the dynamic of friction into an effect reminiscent of restoring force of the spring. Algorithm \ref{Restarting Aggregated Momentum} reinstates the conventional momentum update rules when $k=1$. 

In order to illustrate the effects of RAM on the learning dynamics of algorithms, 
 Figure \ref{trajectory} demonstrates the trajectory of MoMD with negative entropy variant
(MoMWU) in the context of Bias RPS, where the game matrix is defined as $G = \left[\left[0,-1,3\right],\left[1,0,-1\right],\left[-3,1,0\right]\right]$.
 When contrasted with the original MWU (equivalent to setting $\beta=0$ in MoMWU) depicted in Figure \ref{trajectory}(a), the incorporation of negative momentum contributes to the damping of oscillations. Furthermore, higher values of $\beta$ exhibit improved effectiveness in mitigating oscillations within a reasonable range as shown in Figures \ref{trajectory}(b) and \ref{trajectory}(c). Extending the buffer length introduces a distinct periodic updating behavior in the trajectory. The additional attachment point adds a supplementary ``restoring force'', amplifying  resistance in each iteration compared to the standard momentum updating approach with $k=1$. This heightened resistance significantly dampens oscillations, as observed in Figures \ref{trajectory}(d) and \ref{trajectory}(e). The ablation experiments in Appendix \ref{sec: Ablation experiment} further discuss the effect of momentum related to $k$ on the algorithms.

The following theoretical analysis supports the experimental results above.
Assume that $\boldsymbol{z}_*$  is the Nash equilibrium (NE) of the following modified game:
\begin{align}\label{eq: modified game}
    \min_{\boldsymbol x \in \mathcal X} \max_{\boldsymbol y \in \mathcal Y}\boldsymbol x^{\top}\boldsymbol G\boldsymbol y - \frac{\beta}{\eta}D_{\psi}(\boldsymbol x, \boldsymbol x_{att}) + \frac{\beta}{\eta}D_{\psi}(\boldsymbol y, \boldsymbol y_{att}),
\end{align}
where  $ (\boldsymbol x_{att},\boldsymbol y_{att})=\underset{\boldsymbol{z} \in \mathcal{Z}}{\operatorname{argmin}}\left\{\eta\left\langle\boldsymbol{z}, \boldsymbol L_{att} \right\rangle+\psi (\boldsymbol z)\right\}$.
Theorem~\ref{thrm: convergence of MoMWU} establishes that Algorithm~\ref{Restarting Aggregated Momentum}, with a constant learning rate and a profile $\mathcal{L}$ that stores all past snapshots $\boldsymbol{\mu}_t$, converges at an exponentially fast rate to the modified equilibrium $\boldsymbol{z}_*$.
Theorem~\ref{thrm: DualityGap of MoMWU} indicates that the modified equilibrium serves as an $\mathcal{O}(\frac{-\beta}{\eta})$-NE of the original game. This implies that a lower $-\beta$ reduces the duality gap of the modified equilibrium, while a higher $-\beta$ accelerates the convergence rate, introducing a trade-off between convergence speed and the bias of NE attributed to $\beta$. 
Moreover, with an updating attachment point $\boldsymbol{L}_{att}$, Theorem~\ref{thrm: convergence to NE} proves that Algorithm~\ref{Restarting Aggregated Momentum} converges to NE with a sufficiently large $k$. Complete proofs of our theoretical results can be found in Appendix \ref{appendix proof}. 

\begin{theorem}\label{thrm: convergence of MoMWU}
    Let $k=\infty$, i.e., $\boldsymbol{L}_{att}=0$. 
    Then, $\boldsymbol{z}_t$ derived by Algorithm \ref{Restarting Aggregated Momentum} satisfies:
    \begin{align*}
        D_\psi(\boldsymbol z_{*},\boldsymbol z_{t}) \le 
        D_\psi(\boldsymbol z_{*},\boldsymbol z_{0}) \cdot (1+\frac{\beta}{2})^{t},
    \end{align*}
    if $\psi(p)=\langle{p,\ln p}\rangle$, $-\frac{2}{3}<\beta<0$ and $0<\eta\le\frac{\sqrt{-(1+\frac{3}{2}\beta)\beta}}{2}$.
\end{theorem}
\begin{theorem}\label{thrm: DualityGap of MoMWU}
    In the same setup of Theorem \ref{thrm: convergence of MoMWU}, the duality gap for the updated strategy $\boldsymbol{z}_t$ of Algorithm \ref{Restarting Aggregated Momentum} can be bounded as:
    \begin{align*}
        &{DualityGap}(\boldsymbol z_t)
        \\\le&
        \frac{-\beta}{\eta}\cdot\text{diam}(\boldsymbol Z)
        \cdot\Vert{\log\frac{\boldsymbol{z}_{*}}{\boldsymbol{z}_{att}}}\Vert
        +
        \mathcal{O}((1+\frac{\beta}{2})^{\frac{t}{2}}),
    \end{align*}
    where $\text{diam}(\boldsymbol Z)=\sup_{\boldsymbol{z},\boldsymbol{z}'\in\boldsymbol{Z}}\Vert{\boldsymbol{z}-\boldsymbol{z}'}\Vert$,
    ${\boldsymbol{z}_{att}}$ equals $ (\boldsymbol x_{att},\boldsymbol y_{att})=\underset{\boldsymbol{z} \in \mathcal{Z}}{\operatorname{argmin}}\left\{\eta\left\langle\boldsymbol{z}, \boldsymbol L_{att} \right\rangle+\psi (\boldsymbol z)\right\}$. 
\end{theorem}
\begin{theorem}\label{thrm: convergence to NE}
    Algorithm \ref{Restarting Aggregated Momentum} with $\psi(p)=\langle{p,\ln p}\rangle$ converges to the set of Nash equilibria. 
\end{theorem}

We next further extend the concept of negative momentum to RM/$\text{RM}^{+}$.
According to recent work on the interesting connection between FTRL/OMD and RM/$\text{RM}^{+}$~\cite{DBLP:conf/aaai/FarinaKS21}, the regret update $\boldsymbol R^{x}_{t+1} = \left[\boldsymbol R^{x}_{t} + \boldsymbol r(\boldsymbol x_{t})\right]^{+}$ of $\text{RM}^{+}$ can be reformulated as:
\begin{equation}
\boldsymbol{R}_{t+1} =\underset{\hat{\boldsymbol{R}} \in \mathbb{R}^{M}_{+}}{\operatorname{argmin}}\left\{\eta\langle\hat{\boldsymbol{R}}, \boldsymbol r(\boldsymbol x_{t})\rangle+D_{\psi}(\hat{\boldsymbol{R}}, \boldsymbol{R}_{t})\right\}, 
\end{equation}
where $\psi = \frac{1}{2}\|\cdot\|^{2}$ and $\eta=-1$. Therefore, $\text{RM}^{+}$ is intricately linked to OMD instantiated with the non-negative orthant (after thresholding) as the decision set and facing a sequence of loss $(\boldsymbol r(\boldsymbol x_{1}))_{t \geq 1}$. Formally, Lemma \ref{OMD=RM+} draws the relation for the regret in the strategy sequence $\boldsymbol{x}_{1},...\boldsymbol{x}_{T}$ and the regret $\boldsymbol{R}_{1},...,\boldsymbol{R}_{T}$.
\begin{lemma}
\label{OMD=RM+}
\cite{farina2023regretmatchinginstabilityfast} Let $\boldsymbol{x}_{1},...\boldsymbol{x}_{T} \in \Delta^{M}$ be generated as $\boldsymbol x_{t} = \left[\boldsymbol R^{x}_{t}\right]^{+}\big/\|\left[\boldsymbol R^{x}_{t}\right]^{+}\|_{1}$ for some sequence $R^{x}_{1},...R^{x}_{T} \in \mathbb{R}^{M}_{+}$. The regret 
$R_{T, \hat{\boldsymbol{x}}}$ of $\boldsymbol{x}_{1},...\boldsymbol{x}_{T}$ facing a sequence of loss $F(\boldsymbol{x}_{1}),...F(\boldsymbol{x}_{T})$ is equal to $R_{T, \hat{\boldsymbol{R}}}$, i.e., the regret of $\boldsymbol{R}_{1},...,\boldsymbol{R}_{T}$ facing the sequence of loss $\boldsymbol r(\boldsymbol x_{1}),...\boldsymbol r(\boldsymbol x_{T})$, compared against $\hat{\boldsymbol{R}} = \hat{\boldsymbol{x}}$: $R_{T, \hat{\boldsymbol{R}}} = \sum^{T}_{t=1} \langle r(\boldsymbol x_{t}), \boldsymbol{R}_{t} - \hat{\boldsymbol{R}}\rangle$.

\end{lemma}

\begin{algorithm}[tb]
   \caption{Momentum $\text{RM}^{+}$ (Mo$\text{RM}^{+}$)}
   \label{Mo RM+}
\begin{algorithmic}[1]
   \STATE $(\boldsymbol R^{x}_{t}, \boldsymbol R^{y}_{t}) = 0, (\boldsymbol x_{0}, \boldsymbol y_{0}) \in \mathcal{Z}$ 
   \FOR{$t=0$ {\bfseries to} $T$}
   \STATE $\boldsymbol R^{x}_{t+1} = \left[\boldsymbol R^{x}_{t} + \boldsymbol r({x}_{t}) - {\beta(\boldsymbol R^{x}_{att} - \boldsymbol R^{x}_{t})}\right]^{+}$
   \STATE $\boldsymbol x_{t+1} = \left[\boldsymbol R^{x}_{t+1}\right]^{+}\big/\|\left[\boldsymbol R^{x}_{t+1}\right]^{+}\|_{1}$
   \STATE $\boldsymbol R^{y}_{t+1} = \left[\boldsymbol R^{y}_{t} + \boldsymbol r({y}_{t}) - {\beta(\boldsymbol R^{y}_{att} - \boldsymbol R^{y}_{t})}\right]^{+}$
   \STATE $\boldsymbol y_{t+1} = \left[\boldsymbol R^{y}_{t+1}\right]^{+}\big/\|\left[\boldsymbol R^{y}_{t+1}\right]^{+}\|_{1}$
   \IF{$t \ \% \ k = 0$}
    \STATE Update the attachment regret vector \\$\boldsymbol R^{x}_{att} \leftarrow \boldsymbol R^{x}_{t-1}$, $\boldsymbol R^{y}_{att} \leftarrow \boldsymbol R^{y}_{t-1}$ 
   \ENDIF   
   \ENDFOR
\end{algorithmic}
\end{algorithm}

In light of this correspondence, we integrate the updating rules within  RAM into the regret updating framework of $\text{RM}^{+}$, leading to the introduction of variant referred to as \textbf{$\text{Momentum RM}^{+}$} (Algorithm \ref{Mo RM+}).

\textbf{Implementation for Extensive-Form Game (EFG).} Here we briefly  elucidate the process of extending our algorithm to accommodate EFG settings. Generally, there are two approaches to render the resolution of EFG problems computationally feasible. One avenue of research employs the regret decomposition framework~\cite{farina2019online} adopted by the Counterfactual Regret Minimization (CFR)~\cite{zinkevich2007regret} family. This framework is founded on the concept that the global regret of the entire game can be decomposed into the summation  of local regrets associated with each simplex corresponding to a decision node in EFG. We employ Mo$\text{RM}^{+}$ as the local regret minimizer, leading to the algorithm named \textbf{Mo$\text{CFR}^{+}$}.

Another line of research formulates EFG as a bilinear SPP over the \textit{sequence-form} strategy, where $\mathcal{X}$ and $\mathcal{Y}$ represent sequence-form polytopes, which equivalently can be viewed as treeplexes~\cite{hoda2010smoothing}. When employing classical first-order methods such as OMD/FTRL, the class of \textit{dilated distance generating function}~\cite{DBLP:conf/nips/FarinaKS19} becomes crucial for efficient computation and convexity guarantees. We create the DMoMD algorithm by applying the dilated mapping to MoMD. Details on EFGs and our extensive-form algorithms are in Appendix \ref{extensive-form bg}.

\section{Experiments}

In this section, we validate our methods utilizing the exploitability metric under two experimental settings of NFGs and EFGs. 
Detailed game description and hyper-parameters can be found in the Appendix \ref{sec: Experiment Settings}. 

\subsection{Normal-Form Games}

For tabular normal-form games, we conduct experiments on randomly generated NFGs with action dimensions of $25$, $50$, and $75$.  The payoff matrix is drawn from a standard Gaussian distribution with i.i.d realizations using different seed ranging from 0 to 10, and constrains its elements to lie within the range $\left[-1, 1\right]$ after normalization. 
We also verify our algorithms on a $3 \times 3$ game matrix $G = \left[\left[3,0,-3\right],\left[0,3,-4\right],\left[0,0,1\right]\right]$ that has the a unique Nash equilibrium $\left(\boldsymbol{x}^{*}, \boldsymbol{y}^{*}\right) = \left(\left[\frac{1}{12}, \frac{1}{12},\frac{5}{6}\right],\left[\frac{1}{3}, \frac{5}{12},\frac{1}{4}\right]\right)$, which has also been used in \citet{farina2023regretmatchinginstabilityfast} to illustrate the slow ergodic convergence of $\text{RM}^{+}$.

We compare MoMWU and $\text{MoRM}^{+}$ to average-iterate convergent algorithms (RM, $\text{RM}^{+}$) and last-iterate convergent 
algorithms (OMWU, OGDA). We also include the comparison to 
the Magnet Mirror Descent (MMD) 
~\cite{DBLP:conf/iclr/SokotaDKLLMBK23}.
MMD can be adapted into an NE solver by either annealing the amount of regularization over time (MMD-A) or by employing a moving magnet strategy that trails behind the current iteration (MMD-M). 

Note that for $\text{RM}^{+}$ and $\text{MoRM}^{+}$, we employ alternating updates, while other algorithms are kept with their default settings, i.e., simultaneous updates. We plot the average strategy for RM, linear average strategy for $\text{RM}^{+}$ and current strategy for other algorithms. We set the uniform strategy as the initial magnet strategy for MMD-M.

\begin{figure}[ht]
\begin{center}
\centerline{\includegraphics[width=\columnwidth]{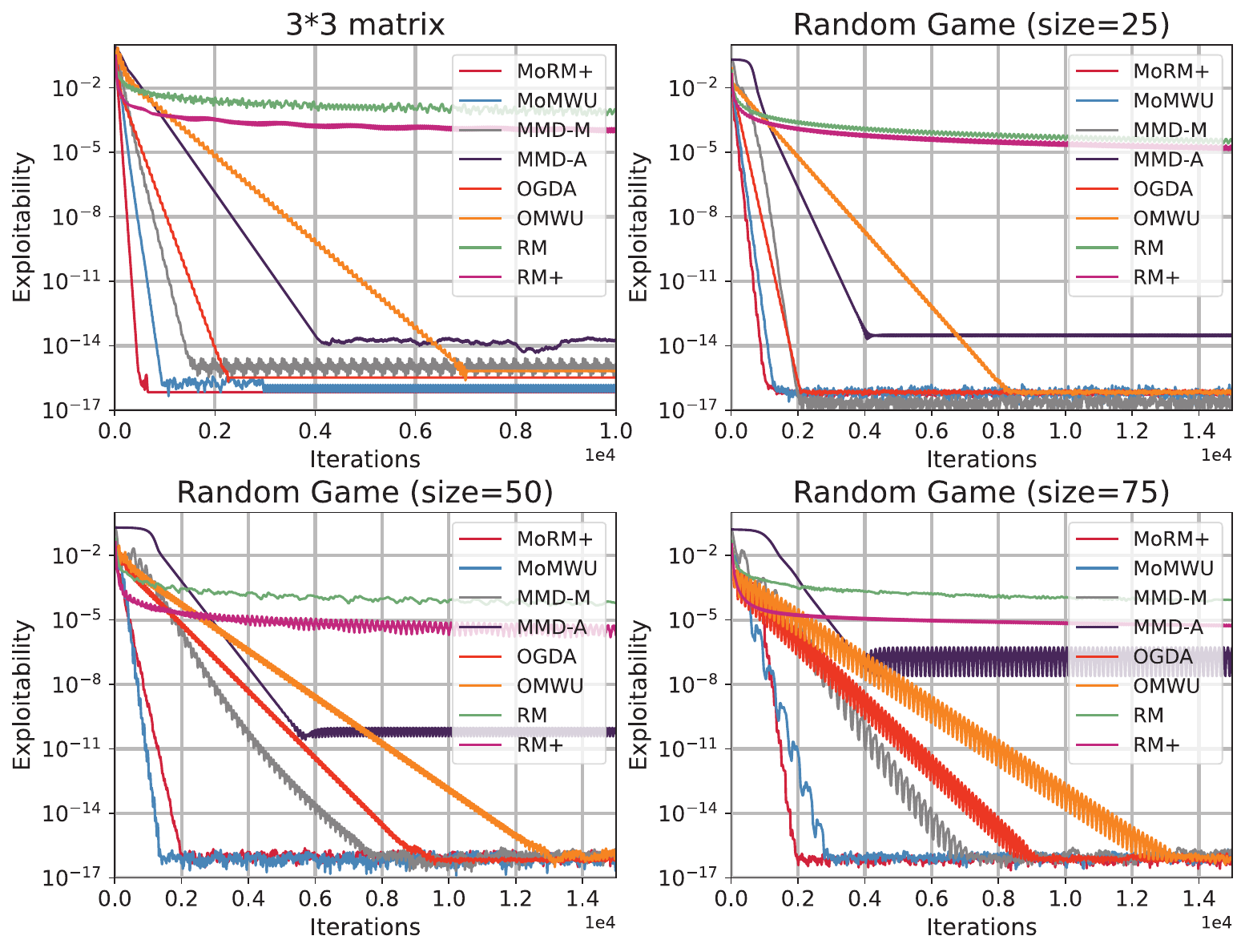}}
\caption{The performance evaluation of the momentum variants and other baseline algorithms in NFGs. In all plots, the x-axis represents the number of iterations for each algorithm, while the y-axis, presented on a logarithmic scale, illustrates the exploitability.}
\label{NFG results}
\end{center}
\end{figure}

The results of NFGs are illustrated in Figure \ref{NFG results}. 
Specifically, in the specific $3 \times 3$ matrix game, $\text{MoRM}^{+}$ effectively overcomes the slow ergodic convergence of $\text{RM}^{+}$ as observed in \cite{cai2023lastiterateconvergencepropertiesregretmatching}, showcasing a notable improvement in convergence speed. This observation underscores the beneficial impact of the momentum technique on enhancing convergence properties. The results of the random matrix game experiments indicate that our momentum variants ($\text{MoRM}^{+}$, MoMWU) consistently demonstrate the fastest convergence rates across all games when compared to their original versions and other rapidly convergent algorithms. Moreover, our algorithms exhibit stability and robustness as the game size increases, whereas optimistic methods (OMWU, OGDA) display a slight oscillation and slower convergence. This observation is further substantiated in subsequent experiments on EFGs.

\subsection{Extensive-Form Games}

In the tabular setting of EFGs, we utilize games implemented in OpenSpiel~\cite{lanctot2020openspielframeworkreinforcementlearning}, encompassing Kuhn Poker, Goofspiel (4/5 cards), Liar’s dice (4/5 sides), and Leduc Poker. We evaluate DMoMD with L2-norm (DMoGDA) and $\text{MoCFR}^{+}$ against average-iterate convergent algorithms (CFR, $\text{CFR}^{+}$). Additionally, predictive update algorithms (OGDA, $\text{PCFR}^{+}$~\cite{DBLP:conf/aaai/FarinaKS21})
along with the regularization algorithms (MMD-M, Reg-CFR, and Reg-DOGDA~\cite{liu2023powerregularizationsolvingextensiveform})
are included in the comparison. Note that for algorithms incorporating  $\text{RM}^{+}$ updating rules ($\text{CFR}^{+}$, $\text{PCFR}^{+}$, and $\text{MoCFR}^{+}$), we employ alternating updates, while other algorithms use simultaneous updates. We present  the average strategy for CFR, linear average strategy for $\text{CFR}^{+}$ and the quadratic average strategy for $\text{PCFR}^{+}$, as recommended in their default setting. Other algorithms are evaluated based on their current strategy. The uniform strategy is set as the initial magnet strategy for MMD-M. 

\begin{figure}[ht]
\begin{center}
\centerline{\includegraphics[width=\columnwidth]{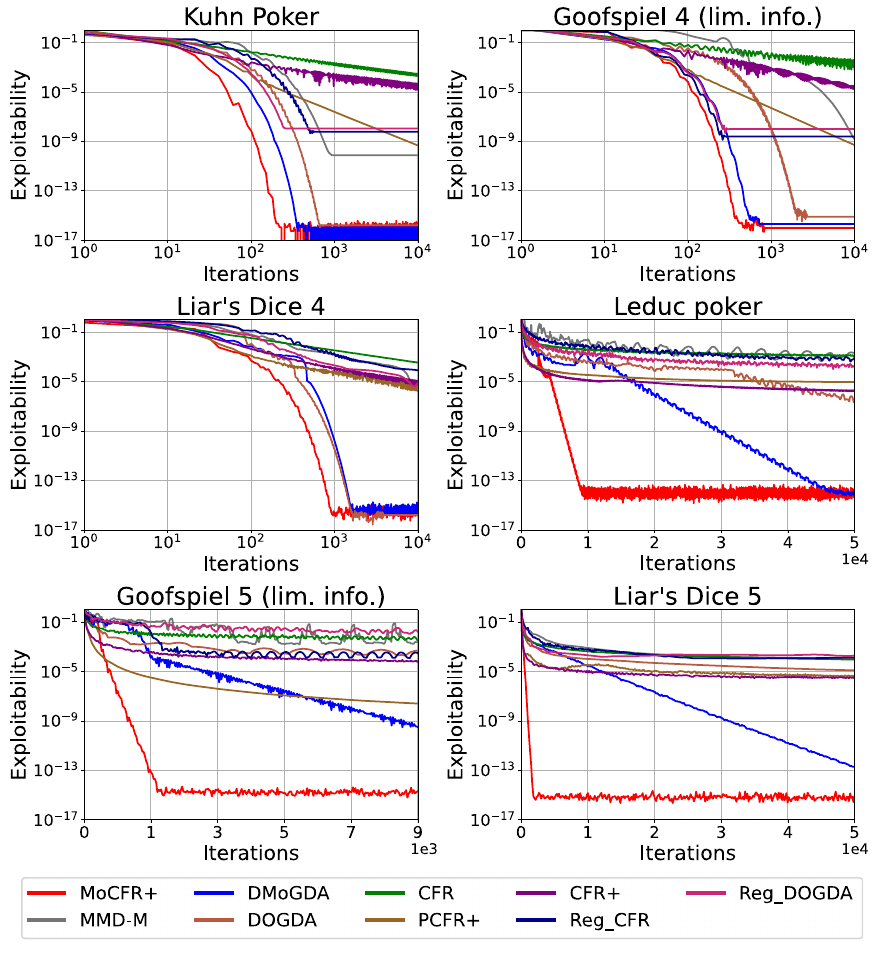}}
\caption{Evaluating the momentum variants and baselines in EFGs. Results are arranged by game sizes. The first three games use a logarithmic x-axis for clearer presentation.}
\label{EFG results}
\end{center}
\end{figure}

Figure \ref{EFG results} illustrates the results in EFGs. Our momentum-augmented algorithms significantly enhance the performance across all games, demonstrating superior empirical convergence rates and substantially lower final exploitability. 
Remarkably, $\text{MoCFR}^{+}$ consistently attains significantly lower exploitability compared to $\text{CFR}^{+}$ and $\text{PCFR}^{+}$. To the best of our knowledge, this marks the first instance where an algorithm surpasses $\text{CFR}^{+}$ performance across various types of games. Additionally, DMoGDA also demonstrates rapid convergence, outpacing all the baselines by a considerable margin. This represents the first occurrence where GDA-type algorithms demonstrate remarkable performance, outperforming SOTA variants of CFR in larger games.

Similar to the observations in NFGs, predictive updates suffer substantial performance degradation with the increasing game size. In contrast, our algorithms can maintain robustness and consistently demonstrate favorable convergence results. An intuitive conjecture posits that, in contract to the friction-like dynamic induced by the negative momentum, the incorporation of predictive updates introduces supplementary second-order information, which can be likened to the application of centripetal force in the learning dynamics, serving to effectively alleviate oscillations~\cite{peng2020training}. However, as the game expands, the escalating non-transitivity poses a significant challenge~\cite{czarnecki2020real}.Methods that utilizes the variations in two consecutive gradient, such as optimistic/extra-gradient, may inadequately furnish the requisite energy to proficiently guide the gradient and disentangle the circular behavior in the learning dynamics. In contrast, our momentum updating framework utilizes a history of past gradients within a period, providing increased energy to mitigate oscillations. This potentially elucidates the robust performance of momentum-augmented algorithms in larger games.

\section{Conclusion}
In this paper, we extend the negative momentum updating paradigm from the unconstrained setting to the constrained setting, seamlessly integrating it with classical algorithms. We formulate the momentum-augmented versions of FTRL, OMD, and \text{RM+}. Leveraging regret decomposition and the dilated distance generate function, we introduce \text{MoCFR+} and DMoMD for solving EFGs. Experiments conducted across numerous benchmark games, demonstrate that momentum-augmented algorithms significantly outperform SOTA algorithms on all tested games.

This work paves the way for some potential research directions. For example, enhancing performance through the combination of predictive and momentum-augmented approaches for improvements is a valuable consideration~\cite{huang2022new}.  Additionally, the integration of other momentum updating paradigms from the literature, such as complex momentum~\cite{lorraine2022complex}, holds promise for achieving faster convergence rates or convergence guarantees over more general settings.

\section{Acknowledgments}
We gratefully acknowledge the support from the National Natural Science Foundation of China (No. 62076259, 62402252), the Fundamental and Applicational Research Funds of Guangdong Province (No. 2023A1515012946), the Fundamental Research Funds for the Central Universities Sun Yat-sen University, and the Pengcheng Laboratory Project (PCL2023A08, PCL2024Y02).


\begin{thebibliography}{64}
\providecommand{\natexlab}[1]{#1}

\bibitem[{Abe et~al.(2024)Abe, Ariu, Sakamoto, and Iwasaki}]{abe2024adaptivelyperturbedmirrordescent}
Abe, K.; Ariu, K.; Sakamoto, M.; and Iwasaki, A. 2024.
\newblock Adaptively Perturbed Mirror Descent for Learning in Games.
\newblock arXiv:2305.16610.

\bibitem[{Abe et~al.(2023)Abe, Ariu, Sakamoto, Toyoshima, and Iwasaki}]{DBLP:journals/corr/abs-2208-09855}
Abe, K.; Ariu, K.; Sakamoto, M.; Toyoshima, K.; and Iwasaki, A. 2023.
\newblock Last-Iterate Convergence with Full and Noisy Feedback in Two-Player Zero-Sum Games.
\newblock arXiv:2208.09855.

\bibitem[{Abernethy, Hazan, and Rakhlin(2008)}]{DBLP:conf/colt/AbernethyHR08}
Abernethy, J.~D.; Hazan, E.; and Rakhlin, A. 2008.
\newblock Competing in the Dark: An Efficient Algorithm for Bandit Linear Optimization.
\newblock In \emph{Annual Conference on Learning Theory}, 263--274. Omnipress.

\bibitem[{Balduzzi et~al.(2018)Balduzzi, Racaniere, Martens, Foerster, Tuyls, and Graepel}]{DBLP:conf/icml/BalduzziRMFTG18}
Balduzzi, D.; Racaniere, S.; Martens, J.; Foerster, J.; Tuyls, K.; and Graepel, T. 2018.
\newblock The mechanics of n-player differentiable games.
\newblock In \emph{International Conference on Machine Learning}, 354--363. PMLR.

\bibitem[{Berry and Shukla(2016)}]{berry2016curl}
Berry, M.; and Shukla, P. 2016.
\newblock Curl force dynamics: symmetries, chaos and constants of motion.
\newblock \emph{New Journal of Physics}, 18(6): 063018.

\bibitem[{Brown and Sandholm(2018)}]{brown2018superhuman}
Brown, N.; and Sandholm, T. 2018.
\newblock Superhuman AI for heads-up no-limit poker: Libratus beats top professionals.
\newblock \emph{Science}, 359(6374): 418--424.

\bibitem[{Burch, Moravcik, and Schmid(2019)}]{burch2019revisiting}
Burch, N.; Moravcik, M.; and Schmid, M. 2019.
\newblock Revisiting CFR+ and alternating updates.
\newblock \emph{Journal of Artificial Intelligence Research}, 64: 429--443.

\bibitem[{Cai et~al.(2023)Cai, Farina, Grand-Clément, Kroer, Lee, Luo, and Zheng}]{cai2023lastiterateconvergencepropertiesregretmatching}
Cai, Y.; Farina, G.; Grand-Clément, J.; Kroer, C.; Lee, C.-W.; Luo, H.; and Zheng, W. 2023.
\newblock Last-Iterate Convergence Properties of Regret-Matching Algorithms in Games.
\newblock arXiv:2311.00676.

\bibitem[{Chavdarova et~al.(2021)Chavdarova, Pagliardini, Stich, Fleuret, and Jaggi}]{chavdarova2021tamingganslookaheadminmax}
Chavdarova, T.; Pagliardini, M.; Stich, S.~U.; Fleuret, F.; and Jaggi, M. 2021.
\newblock Taming GANs with Lookahead-Minmax.
\newblock arXiv:2006.14567.

\bibitem[{Czarnecki et~al.(2020)Czarnecki, Gidel, Tracey, Tuyls, Omidshafiei, Balduzzi, and Jaderberg}]{czarnecki2020real}
Czarnecki, W.~M.; Gidel, G.; Tracey, B.; Tuyls, K.; Omidshafiei, S.; Balduzzi, D.; and Jaderberg, M. 2020.
\newblock Real world games look like spinning tops.
\newblock \emph{Advances in Neural Information Processing Systems}, 33: 17443--17454.

\bibitem[{Daskalakis, Foster, and Golowich(2020)}]{DBLP:conf/nips/DaskalakisFG20}
Daskalakis, C.; Foster, D.~J.; and Golowich, N. 2020.
\newblock Independent Policy Gradient Methods for Competitive Reinforcement Learning.
\newblock In \emph{Advances in Neural Information Processing Systems}.

\bibitem[{Daskalakis and Panageas(2019)}]{DBLP:conf/innovations/DaskalakisP19}
Daskalakis, C.; and Panageas, I. 2019.
\newblock Last-Iterate Convergence: Zero-Sum Games and Constrained Min-Max Optimization.
\newblock In \emph{Innovations in Theoretical Computer Science Conference}, volume 124 of \emph{LIPIcs}, 27:1--27:18. Schloss Dagstuhl - Leibniz-Zentrum f{\"{u}}r Informatik.

\bibitem[{Du et~al.(2017)Du, Chen, Li, Xiao, and Zhou}]{du2017stochastic}
Du, S.~S.; Chen, J.; Li, L.; Xiao, L.; and Zhou, D. 2017.
\newblock Stochastic variance reduction methods for policy evaluation.
\newblock In \emph{International Conference on Machine Learning}, 1049--1058. PMLR.

\bibitem[{Farina et~al.(2023)Farina, Grand-Clément, Kroer, Lee, and Luo}]{farina2023regretmatchinginstabilityfast}
Farina, G.; Grand-Clément, J.; Kroer, C.; Lee, C.-W.; and Luo, H. 2023.
\newblock Regret Matching+: (In)Stability and Fast Convergence in Games.
\newblock arXiv:2305.14709.

\bibitem[{Farina et~al.(2019)Farina, Kroer, Brown, and Sandholm}]{DBLP:conf/icml/FarinaKBS19}
Farina, G.; Kroer, C.; Brown, N.; and Sandholm, T. 2019.
\newblock Stable-Predictive Optimistic Counterfactual Regret Minimization.
\newblock In \emph{International Conference on Machine Learning}, volume~97 of \emph{Proceedings of Machine Learning Research}, 1853--1862. {PMLR}.

\bibitem[{Farina, Kroer, and Sandholm(2019{\natexlab{a}})}]{farina2019online}
Farina, G.; Kroer, C.; and Sandholm, T. 2019{\natexlab{a}}.
\newblock Online convex optimization for sequential decision processes and extensive-form games.
\newblock In \emph{Proceedings of the AAAI Conference on Artificial Intelligence}, volume~33, 1917--1925.

\bibitem[{Farina, Kroer, and Sandholm(2019{\natexlab{b}})}]{DBLP:conf/nips/FarinaKS19}
Farina, G.; Kroer, C.; and Sandholm, T. 2019{\natexlab{b}}.
\newblock Optimistic Regret Minimization for Extensive-Form Games via Dilated Distance-Generating Functions.
\newblock In \emph{Advances in Neural Information Processing Systems}, 5222--5232.

\bibitem[{Farina, Kroer, and Sandholm(2021{\natexlab{a}})}]{DBLP:conf/sigecom/FarinaKS21}
Farina, G.; Kroer, C.; and Sandholm, T. 2021{\natexlab{a}}.
\newblock Better Regularization for Sequential Decision Spaces: Fast Convergence Rates for Nash, Correlated, and Team Equilibria.
\newblock In \emph{{EC} '21: The 22nd {ACM} Conference on Economics and Computation}, 432. {ACM}.

\bibitem[{Farina, Kroer, and Sandholm(2021{\natexlab{b}})}]{DBLP:conf/aaai/FarinaKS21}
Farina, G.; Kroer, C.; and Sandholm, T. 2021{\natexlab{b}}.
\newblock Faster game solving via predictive blackwell approachability: Connecting regret matching and mirror descent.
\newblock In \emph{Proceedings of the AAAI Conference on Artificial Intelligence}, volume~35, 5363--5371.

\bibitem[{Fiez and Ratliff(2021)}]{fiez2021local}
Fiez, T.; and Ratliff, L.~J. 2021.
\newblock Local convergence analysis of gradient descent ascent with finite timescale separation.
\newblock In \emph{International Conference on Learning Representation}.

\bibitem[{Foerster et~al.(2017)Foerster, Chen, Al-Shedivat, Whiteson, Abbeel, and Mordatch}]{Foerster2017LearningWO}
Foerster, J.~N.; Chen, R.~Y.; Al-Shedivat, M.; Whiteson, S.; Abbeel, P.; and Mordatch, I. 2017.
\newblock Learning with Opponent-Learning Awareness.
\newblock In \emph{Adaptive Agents and Multi-Agent Systems}.

\bibitem[{Gidel et~al.(2019)Gidel, Hemmat, Pezeshki, Le~Priol, Huang, Lacoste-Julien, and Mitliagkas}]{DBLP:conf/aistats/GidelHPPHLM19}
Gidel, G.; Hemmat, R.~A.; Pezeshki, M.; Le~Priol, R.; Huang, G.; Lacoste-Julien, S.; and Mitliagkas, I. 2019.
\newblock Negative momentum for improved game dynamics.
\newblock In \emph{The 22nd International Conference on Artificial Intelligence and Statistics}, 1802--1811. PMLR.

\bibitem[{Golowich, Pattathil, and Daskalakis(2020)}]{DBLP:conf/nips/GolowichPD20}
Golowich, N.; Pattathil, S.; and Daskalakis, C. 2020.
\newblock Tight last-iterate convergence rates for no-regret learning in multi-player games.
\newblock \emph{Advances in Neural Information Processing Systems}, 33: 20766--20778.

\bibitem[{Goodfellow et~al.(2020)Goodfellow, Pouget{-}Abadie, Mirza, Xu, Warde{-}Farley, Ozair, Courville, and Bengio}]{DBLP:journals/cacm/GoodfellowPMXWO20}
Goodfellow, I.~J.; Pouget{-}Abadie, J.; Mirza, M.; Xu, B.; Warde{-}Farley, D.; Ozair, S.; Courville, A.~C.; and Bengio, Y. 2020.
\newblock Generative adversarial networks.
\newblock \emph{Communications of the ACM}, 63(11): 139--144.

\bibitem[{Grand-Cl{\'e}ment and Kroer(2024)}]{grand2023solving}
Grand-Cl{\'e}ment, J.; and Kroer, C. 2024.
\newblock Solving optimization problems with Blackwell approachability.
\newblock \emph{Mathematics of Operations Research}, 49(2): 697--728.

\bibitem[{Gulrajani et~al.(2017)Gulrajani, Ahmed, Arjovsky, Dumoulin, and Courville}]{DBLP:conf/nips/GulrajaniAADC17}
Gulrajani, I.; Ahmed, F.; Arjovsky, M.; Dumoulin, V.; and Courville, A.~C. 2017.
\newblock Improved Training of Wasserstein GANs.
\newblock In \emph{Advances in Neural Information Processing Systems}, 5767--5777.

\bibitem[{Hart and Mas-Colell(2000)}]{hart2000simple}
Hart, S.; and Mas-Colell, A. 2000.
\newblock A simple adaptive procedure leading to correlated equilibrium.
\newblock \emph{Econometrica}, 68(5): 1127--1150.

\bibitem[{Heusel et~al.(2017)Heusel, Ramsauer, Unterthiner, Nessler, and Hochreiter}]{DBLP:conf/nips/HeuselRUNH17}
Heusel, M.; Ramsauer, H.; Unterthiner, T.; Nessler, B.; and Hochreiter, S. 2017.
\newblock GANs Trained by a Two Time-Scale Update Rule Converge to a Local Nash Equilibrium.
\newblock In \emph{Advances in Neural Information Processing Systems}, 6626--6637.

\bibitem[{Hoda et~al.(2010)Hoda, Gilpin, Pena, and Sandholm}]{hoda2010smoothing}
Hoda, S.; Gilpin, A.; Pena, J.; and Sandholm, T. 2010.
\newblock Smoothing techniques for computing Nash equilibria of sequential games.
\newblock \emph{Mathematics of Operations Research}, 35(2): 494--512.

\bibitem[{Huang and Zhang(2022)}]{huang2022new}
Huang, K.; and Zhang, S. 2022.
\newblock New first-order algorithms for stochastic variational inequalities.
\newblock \emph{SIAM Journal on Optimization}, 32(4): 2745--2772.

\bibitem[{Korpelevich(1976)}]{korpelevich1976extragradient}
Korpelevich, G.~M. 1976.
\newblock The extragradient method for finding saddle points and other problems.
\newblock \emph{Matecon}, 12: 747--756.

\bibitem[{Kroer et~al.(2019)Kroer, Peysakhovich, Sodomka, and Stier-Moses}]{kroer2019computing}
Kroer, C.; Peysakhovich, A.; Sodomka, E.; and Stier-Moses, N.~E. 2019.
\newblock Computing large market equilibria using abstractions.
\newblock In \emph{ACM Conference on Economics and Computation}, 745--746.

\bibitem[{Kroer et~al.(2020)Kroer, Waugh, K{\i}l{\i}n{\c{c}}-Karzan, and Sandholm}]{kroer2020faster}
Kroer, C.; Waugh, K.; K{\i}l{\i}n{\c{c}}-Karzan, F.; and Sandholm, T. 2020.
\newblock Faster algorithms for extensive-form game solving via improved smoothing functions.
\newblock \emph{Mathematical Programming}, 179(1-2): 385--417.

\bibitem[{Kuhn(1950)}]{kuhn1950simplified}
Kuhn, H.~W. 1950.
\newblock A simplified two-person poker.
\newblock \emph{Contributions to the Theory of Games}, 1: 97--103.

\bibitem[{Lanctot et~al.(2020)Lanctot, Lockhart, Lespiau, Zambaldi, Upadhyay, Pérolat, Srinivasan, Timbers, Tuyls, Omidshafiei, Hennes, Morrill, Muller, Ewalds, Faulkner, Kramár, Vylder, Saeta, Bradbury, Ding, Borgeaud, Lai, Schrittwieser, Anthony, Hughes, Danihelka, and Ryan-Davis}]{lanctot2020openspielframeworkreinforcementlearning}
Lanctot, M.; Lockhart, E.; Lespiau, J.-B.; Zambaldi, V.; Upadhyay, S.; Pérolat, J.; Srinivasan, S.; Timbers, F.; Tuyls, K.; Omidshafiei, S.; Hennes, D.; Morrill, D.; Muller, P.; Ewalds, T.; Faulkner, R.; Kramár, J.; Vylder, B.~D.; Saeta, B.; Bradbury, J.; Ding, D.; Borgeaud, S.; Lai, M.; Schrittwieser, J.; Anthony, T.; Hughes, E.; Danihelka, I.; and Ryan-Davis, J. 2020.
\newblock OpenSpiel: A Framework for Reinforcement Learning in Games.
\newblock arXiv:1908.09453.

\bibitem[{Lanctot et~al.(2009)Lanctot, Waugh, Zinkevich, and Bowling}]{DBLP:conf/nips/LanctotWZB09}
Lanctot, M.; Waugh, K.; Zinkevich, M.; and Bowling, M.~H. 2009.
\newblock Monte Carlo Sampling for Regret Minimization in Extensive Games.
\newblock In \emph{Advances in Neural Information Processing Systems}, 1078--1086. Curran Associates, Inc.

\bibitem[{Lee, Kroer, and Luo(2021)}]{DBLP:conf/nips/LeeKL21}
Lee, C.; Kroer, C.; and Luo, H. 2021.
\newblock Last-iterate Convergence in Extensive-Form Games.
\newblock In \emph{Advances in Neural Information Processing Systems}, 14293--14305.

\bibitem[{Liang and Stokes(2019)}]{DBLP:conf/aistats/LiangS19}
Liang, T.; and Stokes, J. 2019.
\newblock Interaction matters: A note on non-asymptotic local convergence of generative adversarial networks.
\newblock In \emph{International Conference on Artificial Intelligence and Statistics}, 907--915. PMLR.

\bibitem[{Lis{\'{y}}, Lanctot, and Bowling(2015)}]{DBLP:conf/atal/LisyLB15}
Lis{\'{y}}, V.; Lanctot, M.; and Bowling, M.~H. 2015.
\newblock Online Monte Carlo Counterfactual Regret Minimization for Search in Imperfect Information Games.
\newblock In \emph{International Conference on Autonomous Agents and Multiagent Systems}, 27--36. {ACM}.

\bibitem[{Liu et~al.(2023)Liu, Ozdaglar, Yu, and Zhang}]{liu2023powerregularizationsolvingextensiveform}
Liu, M.; Ozdaglar, A.; Yu, T.; and Zhang, K. 2023.
\newblock The Power of Regularization in Solving Extensive-Form Games.
\newblock arXiv:2206.09495.

\bibitem[{Liu et~al.(2022)Liu, Jiang, Li, and Li}]{DBLP:conf/icml/0004J0L22}
Liu, W.; Jiang, H.; Li, B.; and Li, H. 2022.
\newblock Equivalence Analysis between Counterfactual Regret Minimization and Online Mirror Descent.
\newblock In \emph{International Conference on Machine Learning}, volume 162 of \emph{Proceedings of Machine Learning Research}, 13717--13745. {PMLR}.

\bibitem[{Lockhart et~al.(2019)Lockhart, Lanctot, P{\'{e}}rolat, Lespiau, Morrill, Timbers, and Tuyls}]{DBLP:conf/ijcai/LockhartLPLMTT19}
Lockhart, E.; Lanctot, M.; P{\'{e}}rolat, J.; Lespiau, J.; Morrill, D.; Timbers, F.; and Tuyls, K. 2019.
\newblock Computing Approximate Equilibria in Sequential Adversarial Games by Exploitability Descent.
\newblock In \emph{International Joint Conference on Artificial Intelligence}, 464--470.

\bibitem[{Lorraine et~al.(2022)Lorraine, Acuna, Vicol, and Duvenaud}]{lorraine2022complex}
Lorraine, J.~P.; Acuna, D.; Vicol, P.; and Duvenaud, D. 2022.
\newblock Complex momentum for optimization in games.
\newblock In \emph{International Conference on Artificial Intelligence and Statistics}, 7742--7765. PMLR.

\bibitem[{Madras et~al.(2018)Madras, Creager, Pitassi, and Zemel}]{DBLP:conf/icml/MadrasCPZ18}
Madras, D.; Creager, E.; Pitassi, T.; and Zemel, R.~S. 2018.
\newblock Learning Adversarially Fair and Transferable Representations.
\newblock In \emph{International Conference on Machine Learning}, 3381--3390. {PMLR}.

\bibitem[{Mertikopoulos et~al.(2019)Mertikopoulos, Lecouat, Zenati, Foo, Chandrasekhar, and Piliouras}]{DBLP:conf/iclr/MertikopoulosLZ19}
Mertikopoulos, P.; Lecouat, B.; Zenati, H.; Foo, C.; Chandrasekhar, V.; and Piliouras, G. 2019.
\newblock Optimistic mirror descent in saddle-point problems: Going the extra (gradient) mile.
\newblock In \emph{International Conference on Learning Representations}.

\bibitem[{Mescheder, Nowozin, and Geiger(2017)}]{DBLP:conf/nips/MeschederNG17}
Mescheder, L.~M.; Nowozin, S.; and Geiger, A. 2017.
\newblock The Numerics of GANs.
\newblock In \emph{Advances in Neural Information Processing Systems}, 1825--1835.

\bibitem[{Morav{\v{c}}{\'\i}k et~al.(2017)Morav{\v{c}}{\'\i}k, Schmid, Burch, Lis{\`y}, Morrill, Bard, Davis, Waugh, Johanson, and Bowling}]{moravvcik2017deepstack}
Morav{\v{c}}{\'\i}k, M.; Schmid, M.; Burch, N.; Lis{\`y}, V.; Morrill, D.; Bard, N.; Davis, T.; Waugh, K.; Johanson, M.; and Bowling, M. 2017.
\newblock Deepstack: Expert-level artificial intelligence in heads-up no-limit poker.
\newblock \emph{Science}, 356(6337): 508--513.

\bibitem[{Orabona(2023)}]{orabona2023modernintroductiononlinelearning}
Orabona, F. 2023.
\newblock A Modern Introduction to Online Learning.
\newblock arXiv:1912.13213.

\bibitem[{Peng et~al.(2020)Peng, Dai, Zhang, and Cheng}]{peng2020training}
Peng, W.; Dai, Y.-H.; Zhang, H.; and Cheng, L. 2020.
\newblock Training GANs with centripetal acceleration.
\newblock \emph{Optimization Methods and Software}, 35(5): 955--973.

\bibitem[{P{\'e}rolat et~al.(2021)P{\'e}rolat, Munos, Lespiau, Omidshafiei, Rowland, Ortega, Burch, Anthony, Balduzzi, De~Vylder et~al.}]{DBLP:conf/icml/PerolatMLOROBAB21}
P{\'e}rolat, J.; Munos, R.; Lespiau, J.-B.; Omidshafiei, S.; Rowland, M.; Ortega, P.; Burch, N.; Anthony, T.; Balduzzi, D.; De~Vylder, B.; et~al. 2021.
\newblock From poincar{\'e} recurrence to convergence in imperfect information games: Finding equilibrium via regularization.
\newblock In \emph{International Conference on Machine Learning}, 8525--8535. PMLR.

\bibitem[{Polyak(1964)}]{polyak1964some}
Polyak, B.~T. 1964.
\newblock Some methods of speeding up the convergence of iteration methods.
\newblock \emph{Ussr computational mathematics and mathematical physics}, 4(5): 1--17.

\bibitem[{Ross(1971)}]{ross1971goofspiel}
Ross, S.~M. 1971.
\newblock Goofspiel—the game of pure strategy.
\newblock \emph{Journal of Applied Probability}, 8(3): 621--625.

\bibitem[{Sch{\"{a}}fer and Anandkumar(2019)}]{DBLP:conf/nips/SchaferA19}
Sch{\"{a}}fer, F.; and Anandkumar, A. 2019.
\newblock Competitive Gradient Descent.
\newblock In \emph{Advances in Neural Information Processing Systems}, 7623--7633.

\bibitem[{Shi et~al.(2019)Shi, Du, Su, and Jordan}]{DBLP:conf/nips/ShiDSJ19}
Shi, B.; Du, S.~S.; Su, W.~J.; and Jordan, M.~I. 2019.
\newblock Acceleration via Symplectic Discretization of High-Resolution Differential Equations.
\newblock In \emph{Advances in Neural Information Processing Systems}.

\bibitem[{Sinha, Namkoong, and Duchi(2018)}]{DBLP:conf/iclr/SinhaND18}
Sinha, A.; Namkoong, H.; and Duchi, J.~C. 2018.
\newblock Certifying Some Distributional Robustness with Principled Adversarial Training.
\newblock In \emph{International Conference on Learning Representations}.

\bibitem[{Sokota et~al.(2023)Sokota, D'Orazio, Kolter, Loizou, Lanctot, Mitliagkas, Brown, and Kroer}]{DBLP:conf/iclr/SokotaDKLLMBK23}
Sokota, S.; D'Orazio, R.; Kolter, J.~Z.; Loizou, N.; Lanctot, M.; Mitliagkas, I.; Brown, N.; and Kroer, C. 2023.
\newblock A Unified Approach to Reinforcement Learning, Quantal Response Equilibria, and Two-Player Zero-Sum Games.
\newblock In \emph{International Conference on Learning Representations,}.

\bibitem[{Southey et~al.(2012)Southey, Bowling, Larson, Piccione, Burch, Billings, and Rayner}]{southey2012bayesbluffopponentmodelling}
Southey, F.; Bowling, M.~P.; Larson, B.; Piccione, C.; Burch, N.; Billings, D.; and Rayner, C. 2012.
\newblock Bayes' Bluff: Opponent Modelling in Poker.
\newblock arXiv:1207.1411.

\bibitem[{Tammelin(2014)}]{tammelin2014solvinglargeimperfectinformation}
Tammelin, O. 2014.
\newblock Solving Large Imperfect Information Games Using CFR+.
\newblock arXiv:1407.5042.

\bibitem[{v.~Neumann(1928)}]{v1928theorie}
v.~Neumann, J. 1928.
\newblock Zur theorie der gesellschaftsspiele.
\newblock \emph{Mathematische annalen}, 100(1): 295--320.

\bibitem[{Vlatakis{-}Gkaragkounis, Flokas, and Piliouras(2019)}]{DBLP:conf/nips/Vlatakis-Gkaragkounis19a}
Vlatakis{-}Gkaragkounis, E.; Flokas, L.; and Piliouras, G. 2019.
\newblock Poincar{\'{e}} Recurrence, Cycles and Spurious Equilibria in Gradient-Descent-Ascent for Non-Convex Non-Concave Zero-Sum Games.
\newblock In \emph{Advances in Neural Information Processing Systems}, 10450--10461.

\bibitem[{Warmuth, Jagota et~al.(1997)}]{warmuth1997continuous}
Warmuth, M.~K.; Jagota, A.~K.; et~al. 1997.
\newblock Continuous and discrete-time nonlinear gradient descent: Relative loss bounds and convergence.
\newblock In \emph{International Symposium on Artificial Intelligence and Mathematics}, volume 326. Citeseer.

\bibitem[{Wei et~al.(2021)Wei, Lee, Zhang, and Luo}]{DBLP:conf/iclr/WeiLZL21}
Wei, C.; Lee, C.; Zhang, M.; and Luo, H. 2021.
\newblock Linear Last-iterate Convergence in Constrained Saddle-point Optimization.
\newblock In \emph{International Conference on Learning Representations}.

\bibitem[{Zhang and Wang(2021)}]{zhang2021suboptimality}
Zhang, G.; and Wang, Y. 2021.
\newblock On the suboptimality of negative momentum for minimax optimization.
\newblock In \emph{International Conference on Artificial Intelligence and Statistics}, 2098--2106. PMLR.

\bibitem[{Zinkevich et~al.(2007)Zinkevich, Johanson, Bowling, and Piccione}]{zinkevich2007regret}
Zinkevich, M.; Johanson, M.; Bowling, M.; and Piccione, C. 2007.
\newblock Regret minimization in games with incomplete information.
\newblock \emph{Advances in Neural Information Processing Systems}, 20.

\end{thebibliography}

\renewcommand\thesubsection{\Alph{subsection}}

\onecolumn
\section{Appendix}
\setcounter{secnumdepth}{2}
\subsection{Momentum-augmented Algorithms in Extensive-form Games}
\label{extensive-form bg}

\subsubsection{Sequential Decision Process}

We follows the work by \citet{DBLP:conf/icml/FarinaKBS19} and \citet{DBLP:conf/icml/0004J0L22} and introduces the sequence form for sequential decision process, which effectively captures the players’ decision process in two-player zero-sum EFGs. 
A sequential decision process consists of two kinds of points: \textit{decision points} and \textit{observation points}. The set of decision points is denoted by $J$, while the set of observation points is denoted as $K$. At each decision point $j \in J$, the player is assigned to make a \textit{(local) decision} $\hat{\boldsymbol{x}}_j \in \Delta^{n_j}$, where $\Delta^{n_j}$ is a simplex over the action set $A_j$ and $n_j = |A_j|$.
The set of all actions is denoted by $A$. The combination of $\hat{\boldsymbol{x}}_j$ across all decision points constitute a (\textit{behaviour form}) \textit{strategy}. Let $\hat{\boldsymbol{x}}_j[a]$
be the probability of choosing \textit{action} $a \in A_j$. Each action leads the player to an observation point $k \in K$, denoted
by $k = \rho(j, a)$. At each observation point, the player acquires a signal $s \in S_k$. Following signal observation, the player transitions to another decision point $j' \in J$, written as $j' = \rho(k, s)$.

Given a specific action $a$ at $j$, the set of possible decision points that player may next face is denoted by $C_{j,a}$. It can be an empty set if no more actions are taken after $j,a$. Formally, $C_{j,a} = \{\rho(\rho(j, a), s) : s \in S_{\rho(j,a)}\}$. If $j' \in C_{j,a}$, we say that $j'$ is a child of $j$, and $j$ is the parent of $j'$. $C_{j,a}$ can be thought of as representing the different decision points that an player may face after taking action $a$ and then making an observation on which she can condition her next action choice. The set of all descending decision points of $j$ (including $j$) can be denoted recursively as $C_{\downarrow j} = \{j\} \cup \bigcup_{j' \in C_{j,a}, a \in A_j} C_{\downarrow j'}$. We assume that the process forms a tree. In other words, $C_{j,a} \cap C_{j',a'} = \emptyset$ for any $(j, a) \neq (j', a')$. This is equivalent to the perfect-recall assumption in EFGs.  In practical scenarios, there might be a requirement for multiple root decision points, such as when modeling various starting hands in card games. The inclusion of multiple root decision points can be addressed by introducing a dummy root decision point with a single action. Henceforth, we assume a sequential decision process always starts from a decision point, named the root decision point and denoted by $o$.

\textbf{Sequence Form for Sequential Decision Processes} The expected loss for a given behaviour-form strategy, as defined before, exhibits non-linearity in the vector of decision variables $\left(\hat{\boldsymbol{x}}_{j}\right)_{j\in J}$. This non-linearity arises from the product $\pi_j$ of probabilities associated with all actions along the path from the root to $j$. A representation of this decision space that upholds linearity is known as \textit{sequence-form}. 
A \textit{sequence} is a series of $(j,a)$ starting from the root in a sequential decision process. 
In sequence-form representation, a strategy is the combination of the probabilities of playing each sequence. This is equivalent to adjusting the behavior-form strategy space at a generic decision point $j \in J$ based on the decision variable associated with the last action in the path from the root of the process to $j$. 
This paper characterize the sequence-form strategy space using a \textit{treeplex}~\cite{hoda2010smoothing}, following the construction method detailed in the work~\cite{DBLP:conf/nips/FarinaKS19}.
Formally, let $\mathcal{X}$ denote the sequence-form strategy space, and $\boldsymbol{x}$ represent a strategy in $\mathcal{X}$. The space $\mathcal{X}$ is recursively defined: at each decision point $j \in J$, $\mathcal{X}_{j,a} = \prod_{j'\in C_{j,a}} \mathcal{X}_{j'}$ (cartesian product). $\mathcal{X}_{j}$ is defined as:
\begin{equation}
\mathcal{X}_{j}=\left\{  \left(\hat{\boldsymbol{x}}_{j}, \hat{\boldsymbol{x}}_{j}\left[a_{1}\right] \boldsymbol{x}_{j, a_{1}} \ldots, \hat{\boldsymbol{x}}_{j}\left[a_{n_{j}}\right] \boldsymbol{x}_{j, a_{n_{j}}}\right):
\hat{\boldsymbol{x}}_{j} \in \Delta^{n_{j}}, \boldsymbol{x}_{j, a_{1}} \in \mathcal{X}_{j, a_{1}}, \ldots, \boldsymbol{x}_{j, a_{n_{j}}} \in \mathcal{X}_{j, a_{n_{j}}}\right\},
\label{recursive sequence-form}
\end{equation}
where $\hat{\boldsymbol{x}}_j \in \Delta^{n_j}$ and $(a_{1},...,a_{n_{j}}) = A_{j}$. Let $\mathcal{X} = \mathcal{X}_{0}$. In this formulation, the value of a particular action represents the probability of playing the
whole sequence of actions from the root to that action. Crucially, both $\mathcal{X}$ and all $\mathcal{X}_{j}$ are treeplexes, ensuring convexity and compactness~\cite{hoda2010smoothing}. Intuitively, $\mathcal{X}_{j}$ represents the sequence-form strategy space of the sub-sequential decision process that initiates from decision point $j$.

Given a concatenated vector $\boldsymbol{z} = (\boldsymbol{z}_{j_{1}}, \ldots, \boldsymbol{z}_{j_{m}}) \in \mathbb{R}^{\sum_{k=1}^{m} n_{j_k}}$, e.g., $\boldsymbol{x} \in \mathcal{X}$, we may need to extract a sub-vector related to decision point $j$ only. Formally, let $\boldsymbol{z}[j]$ represent the $n_j$ entries related to decision point $j$, and $\boldsymbol{z}[j, a]$ represent the entry corresponding to $(j, a)$. Additionally, let $\boldsymbol{z}[\downarrow j]$ denote the sub-vector corresponding to the decision points in $C_{\downarrow j}$. For any vector $\boldsymbol{z} \in \mathbb{R}^{n_j}$, e.g., a decision $\hat{x}_{j}$ at $j$, $\boldsymbol{z}[a]$ represents the entry corresponding to $a$. Let $p_j$ denote the pair $(j', a')$ such that $j \in C_{j',a'}$. Then, $\boldsymbol{x}[p_j] = \boldsymbol{x}[j', a']$ is the probability of reaching decision point $j$. Note that $p_o$ is undefined, and for convenience, we let $\boldsymbol{x}[p_o] = 1$. Based on the above definitions, there exist straightforward mappings among $\boldsymbol{x} \in \mathcal{X}$, $\boldsymbol{x}_j \in \mathcal{X}_j$, and a decision $\hat{\boldsymbol{x}}_j \in \Delta^{n_j}$. Formally, we redefine $\boldsymbol{x}_j = \boldsymbol{x}[\downarrow j] / \boldsymbol{x}[p_j]$ and $\hat{\boldsymbol{x}}_j = \boldsymbol{x}[j] / \boldsymbol{x}[p_j]$ for any $j \in J$. In the remainder of the paper,standard symbols like $\boldsymbol{x}$ and $\boldsymbol{x}_j$ represent variables related to the sequence-form space, while symbols with hats, such as $\hat{\boldsymbol{x}}_j$, represent variables related to local decision points.

By leveraging the sequence-form representation, finding a Nash equilibrium in a two-player zero-sum EFG with perfect recall can be framed as a bilinear SPP akin to the form illustrated in the SPP (\ref{bilinear form}), where the $\mathcal{X}$ and $\mathcal{Y}$ represents sequence-form strategy spaces of the sequential decision processes faced by the two players, and $G$ is a sparse matrix encoding the leaf payoffs of the game. To ensure conformity with the existing literature, let $\boldsymbol{l}^{t} = \boldsymbol{G}\boldsymbol{y}_{t}$, which is equal to $\boldsymbol f_{t}$ in the main text.
\subsubsection{Regret Decomposition Framework}
To efficiently solve the bilinear SPP (\ref{bilinear form}) with the integration of the sequence-form strategy, the adoption of regret decomposition stands out as a crucial element. It serves as a cornerstone for algorithms like CFR~\cite{zinkevich2007regret}, contributing to the resolution of large imperfect information games. It has been proven that the total regret for each player after $T$ iterations in the complete game is constrained by the sum of local (counterfactual) regrets and is bounded by $\mathcal{O}(\sqrt{T})$. Consequently, the average strategies will converge to a Nash equilibrium at a rate of $\mathcal{O}(1 / \sqrt{T})$. 
The central concept of CFR, namely counterfactual regret, can be reformulated based on the sequence-form representation~\cite{DBLP:conf/icml/FarinaKBS19}. Formally, 
given a strategy $\boldsymbol{x}_t \in \mathcal{X}$ and a loss $\boldsymbol l_t \in \mathbb{R}^{M}$, CFR constructs a counterfactual loss $\hat{\boldsymbol l}{}_j^t \in \mathbb{R}^{n_j}$ recursively for each $j \in J$:
\begin{equation}
    \hat{\boldsymbol{l}}{}_{j}^{t}[a]=\boldsymbol{l}^{t}[j, a]+\sum_{j^{\prime} \in C_{j, a}}\left\langle\hat{\boldsymbol{l}}{}_{j^{\prime}}^{t}, \hat{\boldsymbol{x}}_{j^{\prime}}^{t}\right\rangle .
    \label{counterfactual loss}
\end{equation}
By definition, $\langle\hat{\boldsymbol{l}}{}_{o}^{t}, \hat{\boldsymbol{x}}_{o}^{t}\rangle = \langle\hat{\boldsymbol{l}}{}^{t}, \hat{\boldsymbol{x}}^{t}\rangle$. Define the cumulative counterfactual loss as $\hat{\boldsymbol{L}}{}_{j}^{t}=\sum^{t}_{k=1}\hat{\boldsymbol{l}}{}_{j}^{k}$. Following Equation (\ref{counterfactual loss}), the cumulative counterfactual loss at local decision point can be written as:
\begin{equation}
    \hat{\boldsymbol{L}}{}_{j}^{t}[a]=\boldsymbol{L}^{t}[j, a]+\sum_{j^{\prime} \in C_{j, a}}\left(\sum_{k=1}^{t}\left\langle\hat{\boldsymbol{l}}{}_{j^{\prime}}^{k}, \hat{\boldsymbol{x}}_{j^{\prime}}^{k}\right\rangle \right).
    \label{cumulative counterfactual loss}
\end{equation}
Like the definitions in Section (\textbf{Preliminaries}), the instantaneous counterfactual regret vector can be defined as $\hat{\boldsymbol r}{}_{j}(\hat{\boldsymbol x}{}_{j}^{t}) = \langle\hat{\boldsymbol l}{}_{j}^{t}, \hat{\boldsymbol x}{}_{j}^{t}\rangle \cdot \boldsymbol1_{L}-\hat{\boldsymbol l}{}_{j}^{t}$, and the accumulative counterfactual regret $\hat{\boldsymbol R}{}^{t}_{j} = \sum_{\tau=0}^{t-1} \hat{\boldsymbol r}{}_{j}(\hat{\boldsymbol x}{}_{j}^{\tau})$. Denote the upper bound $\hat{R}{}^{t}_{j} =\max_{a \in A_{j}} \hat{\boldsymbol R}{}^{t}_{j}[a]$. The main theorem of CFR can be formulated as $R^{T} \leq \sum_{j \in J}[\hat{R}{}^{t}_{j}]^{+}$, as stated in~\cite{zinkevich2007regret}. Hence, by selecting a sequence of strategies that yield sublinear counterfactual regret in each information set, one can attain sublinear total regret. CFR achieves this by applying any
\textit{regret minimizer} learning algorithm, such as RM or RM+, operating on counterfactual regret locally at the decision point. The CFR/\text{CFR+} algorithms proceeds in the loop of following three operations: 
\begin{equation}
\begin{array}{l}
\text{Step I:} \qquad \qquad \ \ \ 
\hat{\boldsymbol{l}}{}_{j}^{t}[a] \leftarrow \boldsymbol{l}^{t}[j, a]+\sum_{j^{\prime} \in C_{j, a}}\langle\hat{\boldsymbol{l}}{}_{j^{\prime}}^{t}, \hat{\boldsymbol{x}}{}_{j^{\prime}}^{t}\rangle,\\
\text{Step II:} \qquad \qquad 
\hat{\boldsymbol{R}}{}_{j}^{t+1} \leftarrow \hat{\boldsymbol{R}}{}_{j}^{t}+ \hat{\boldsymbol r}{}_{j}(\hat{\boldsymbol x}{}_{j}^{t}),(\text{RM}) \quad \quad \quad
\hat{\boldsymbol{R}}{}_{j}^{t+1} \leftarrow [\hat{\boldsymbol{R}}{}_{j}^{t}+ \hat{\boldsymbol r}{}_{j}(\hat{\boldsymbol x}{}_{j}^{t})]^{+},(\text{RM+}) \\
\text{Step III:} \qquad \qquad 
\hat{\boldsymbol{x}}_{j}^{t+1} \leftarrow[\hat{\boldsymbol{R}}{}_{j}^{t+1}]^{+} /\|[\hat{\boldsymbol{R}}{}_{j}^{t+1}]^{+}\|_{1}.
\end{array}
\end{equation}
Note that the counterfactual loss $(\hat{\boldsymbol{l}}{}_{j}^{t})_{j \in J}$ can be calculated bottom-up. A direct implication arises when we instantiate step (II) with \text{MoRM+} (Algorithm \ref{Mo RM+}), the \text{MoCFR+} procedure can then be derived as outlined below:
\begin{equation}
\begin{array}{l}
\text{Step I:} \qquad \qquad \ \ \ 
\hat{\boldsymbol{l}}{}_{j}^{t}[a] \leftarrow \boldsymbol{l}^{t}[j, a]+\sum_{j^{\prime} \in C_{j, a}}\langle\hat{\boldsymbol{l}}{}_{j^{\prime}}^{t}, \hat{\boldsymbol{x}}{}_{j^{\prime}}^{t}\rangle,\\
\text{Step II:} \qquad \qquad 
\hat{\boldsymbol{R}}{}_{j}^{t+1} \leftarrow [\hat{\boldsymbol{R}}{}_{j}^{t}+ \hat{\boldsymbol r}{}_{j}(\hat{\boldsymbol x}{}_{j}^{t}) - {\beta(\hat{\boldsymbol{R}}{}_{j}^{att} - \hat{\boldsymbol{R}}{}_{j}^{t})}]^{+},\quad (\text{MoRM+}) \quad \quad \quad \quad \quad\\
\text{Step III:} \qquad \qquad 
\hat{\boldsymbol{x}}_{j}^{t+1} \leftarrow[\hat{\boldsymbol{R}}{}_{j}^{t+1}]^{+} /\|[\hat{\boldsymbol{R}}{}_{j}^{t+1}]^{+}\|_{1},\\
\text{Step IV:} \qquad \qquad 
\hat{\boldsymbol{R}}{}_{j}^{att} \leftarrow \hat{\boldsymbol{R}}{}_{j}^{t-1} \quad \quad\text{if \quad $t\ \%\ k = 0$}.
\end{array}
\end{equation}

\text{MoCFR+} can be readily implemented by transforming the original CFR/\text{CFR+} through the storage of an additional accumulative attachment regret vector $(\hat{\boldsymbol{R}}{}_{j}^{att})_{j \in J}$. 

\subsubsection{Dilated Distance Generate Function}
Besides regret decomposition framework, in order to apply first-order methods in the  bilinear structure of the sequence-form problem, a crucial aspect is the selection of a suitable mirror map for treeplex $\mathcal{X}$ and $\mathcal{Y}$. One such choice is within the class of dilated Distance Generating Functions (DGF):
\begin{equation}
\psi^{\text{dil}}(\boldsymbol{x}) = \sum_{j \in J} \alpha_{j}\boldsymbol{x}[p_{j}]\psi_{j}(\hat{\boldsymbol{x}}_{j}),
\label{dilated entropy}
\end{equation}
where $(\alpha_{j})_{j \in J} > 0$ are decision-point-wise weights and $\psi_{j}$ is a local DGF for the simplex $\hat{\boldsymbol{x}}_{j}$.The essence behind Equation (\ref{dilated entropy}) is that it represents the construction of a dilated DGF through the summation of suitable local DGFs for each decision point. In this process, each local DGF undergoes dilation by the parent variable associated with the respective decision point. Although there is existing literature~\cite{DBLP:conf/sigecom/FarinaKS21,kroer2020faster} discussing the proper configuration of $(\alpha_{j})_{j \in J}$ to achieve enhanced strong convexity properties, it remains orthogonal to the focus of our work. In the subsequent discussion and experimental analyses, we adopt a simplified approach by setting $\alpha_{j} = 1$ for all decision points, following the methodology employed in~\cite{DBLP:conf/nips/FarinaKS19}.

\textbf{Dilated Momentum-augmented Gradient Descent Ascent (DMoGDA)}. Define the dilated squared Euclidean norm regularizer
$\psi^{\text{dil}}_{l_{2}}$ with $\psi_{j}$ being the vanilla squared Euclidean norm $\psi_{j}(\hat{\boldsymbol{x}}_{j}) = \frac{1}{2}\|\hat{\boldsymbol{x}}_{j}\|^{2}$. We call the MoMD (Equation (\ref{MMDm})) with regularizer $\psi^{\text{dil}}_{l_{2}}$ Dilated Momentum-augmented Gradient Descent Ascent (DMoGDA).

\textbf{Dilated Momentum-augmented Multiplicative Weight Update (DMoMWU)}. Define the dilated entropy regularizer
$\psi^{\text{dil}}_{\text{ent}}$ with $\psi_{j}$ being the vanilla entropy $\psi({\hat{\boldsymbol{x}}_{j}}) = \hat{\boldsymbol{x}}_{j}\log  \hat{\boldsymbol{x}}_{j}$. We call the MoMD with regularizer $\psi^{\text{dil}}_{\text{ent}}$ Dilated Momentum-augmented Multiplicative Weight Update (DMoMWU). 

The DmoMD updating rules can be written as follows:
\begin{equation}
\begin{aligned}
    \boldsymbol \mu_{t}&=\beta \boldsymbol \mu_{t-1} - F(\boldsymbol{x}_{t})\\
    \boldsymbol{x}_{t+1} &=\underset{\boldsymbol{x} \in \mathcal{X}}{\operatorname{argmin}}\left\{\eta\left\langle\boldsymbol{x}, \boldsymbol -\boldsymbol \mu_{t}\right\rangle+D_{\psi^{\text{dil}}}\left(\boldsymbol{x}, \boldsymbol{x}_{t}\right)\right\}\\
    &=\underset{\boldsymbol{x} \in \mathcal{X}}{\operatorname{argmin}}\left\langle\boldsymbol{x}, \boldsymbol -\eta\boldsymbol \mu_{t} - \nabla\psi^{\text{dil}}(\boldsymbol{x}_{t})\right\rangle + \psi^{\text{dil}}(\boldsymbol{x})\\
    &=\underset{\boldsymbol{x} \in \mathcal{X}}{\operatorname{argmin}} \sum_{j \in J}\left\langle\boldsymbol{x}[j], \boldsymbol -\eta\boldsymbol \mu_{t}[j] - \nabla\psi^{\text{dil}}(\boldsymbol{x}_{t})[j]\right\rangle + \boldsymbol{x}[p_{j}]\psi_{j}(\hat{\boldsymbol{x}}_{j}) \\
    &=\underset{\boldsymbol{x} \in \mathcal{X}}{\operatorname{argmin}} \sum_{j \in J}\boldsymbol{x}[p_{j}] \left(\left\langle\hat{\boldsymbol{x}}_{j}, \boldsymbol -\eta\boldsymbol \mu_{t}[j] - \nabla\psi^{\text{dil}}(\boldsymbol{x}_{t})[j]\right\rangle + \psi_{j}(\hat{\boldsymbol{x}}_{j})\right).
    \label{DmoMD}
\end{aligned}
\end{equation}
The dilated updates of Equation (\ref{DmoMD}) can be computed in closed-form, starting from decision points $j$ without any children and progressing upwards in the sequential decision process. This process follows a reminiscent "bottom-up" paradigm. For a more comprehensive understanding of these types of updates, please refer to~\cite{DBLP:conf/nips/LeeKL21, DBLP:conf/icml/0004J0L22}.

\subsection{Experiment Settings}\label{sec: Experiment Settings}
\subsubsection{Game Description}
\textbf{Kuhn poker} is a standard benchmark in the EFG-solving community~\cite{kuhn1950simplified}. 
In the game of Kuhn Poker, each participant initially contributes an ante valued at 1 unit to the central pot. Subsequently, every player receives a single card from a deck containing three distinct cards, with these cards being privately dealt. A singular round of betting then ensues, characterized by the following sequential dynamics:
Initially, Player 1 must decide between two actions: to either check or to place a bet of 1 unit. 
\begin{itemize}
    \item If Player 1 checks, Player 2 can check or raise 1.
    \begin{itemize}
        \item If Player 2 checks, a showdown occurs; if Player 2 raises, Player 1 can fold or call.
        \begin{itemize}
            \item If Player 1 folds, Player 2 takes the pot; if Player 1 calls, a showdown occurs.
        \end{itemize}
    \end{itemize}
    \item  If Player 1 bets, Player 2 can fold or call.
    \begin{itemize}
        \item If Player 2 folds, Player 1 takes the pot; if Player 2 calls, a showdown occurs.
    \end{itemize}
\end{itemize}
During a showdown, the player possessing the higher-valued card emerges as the victor and consequently seizes the pot, thereby concluding the game instantaneously.

\textbf{Leduc poker} represents another established benchmark within the EFG-solving community~\cite{southey2012bayesbluffopponentmodelling}. 
This game operates with a deck comprising two suits of cards, each appearing twice. Structurally, Leduc Poker consists of two distinct rounds.
In the first round, each player contributes an ante of 1 unit to the central pot and receives a single private card. Subsequently, a round of betting unfolds, commencing with Player 1's action. Notably, each player is restricted to a maximum of two bets throughout this round.
Then, following the initial betting phase, a card is unveiled face up, inaugurating the second round of betting, with the same dynamics described above. It is imperative to note that all bets placed during the initial round hold a value of 1 unit, while those in the subsequent round are set at 2 units.
After the two betting round, if one of the players has a pair with the public card, that player wins the pot. Otherwise, the player with the higher card wins the pot.

\textbf{Goofspiel 5} stands as a prominent benchmark in EFG~\cite{ross1971goofspiel}. 
This game involves three distinct suits of cards, each comprising 5 ranks. 
At the beginning of the game, each player is dealt with one suit of private cards, and another suit of cards is served as the prize and kept face down on the desk. 
During each round of gameplay, the topmost card from the prize pool is revealed. 
Then, the players are asked to select a card from their respective hands, which they reveal simultaneously. 
The player who has a higher-ranked card wins the prize card, while in the event of a tie in ranks, the prize card is evenly split between the players.
After all prize cards are revealed, each player's score is determined by the cumulative sum of the ranks of the prize cards they have won. 
Subsequently, the payoff structure is as follows: the player with the higher score receives a payoff of +1, while their opponent receives a payoff of -1. In cases where both players have identical scores, the payoffs are distributed equally, resulting in a payoff of (0.5, 0.5).

\textbf{Goofspiel 4} is the same as Goofspiel 5, except that Goofspiel 4 has three suits of cards with each suit comprises 4 ranks.

\textbf{Goofspiel with limited-information} is a variant of the Goofspiel game in \citet{DBLP:conf/nips/LanctotWZB09}. 
In this game, players abstain from directly revealing their cards to one another. Instead, they present their respective cards to a pair umpire, whose role entails adjudicating which player has played the highest-ranking card, thereby deserving the prize card. 
In case of tie, the umpire directs the players to discard the prize card, similar to the procedure in the traditional Goofspiel game.

\textbf{Liar’s dice 4} is another standard benchmark in EFG-solving community~\cite{DBLP:conf/atal/LisyLB15}.
At the beginning of the game, each of the two players initially privately rolls an unbiased 4-face die.
Then, the first player begins bidding, articulating it in the form of Quantity-Value.
For example, a claim of "1-2" signifies the player's belief that there exists one die exhibiting a face value of 2.
Then, the second player can make a higher bid, or to call the previous bidder a “liar”. 
A bid is higher than the previous one if it has a higher Quantity with any Value or the same Quantity with a higher Value. 
When the player calls a liar, all dice are revealed.
If the bid is valid, the last bidder wins and receives a payoff of +1, while the challenger obtains a negative payoff of -1. 
Otherwise, the challenger wins the payoff of +1 and the last bidder loses the game with the reward -1. 

\textbf{Liar’s dice 5} is the same as Liar’s dice 4, except that Liar’s dice 5 rolls an unbiased 5-face die.

\subsubsection{Hyper-parameter Settings}
The hyper-parameters for algorithms we conducted in NFGs are listed in Table \ref{table: hyper-para of in NFGs}, and the
hyper-parameters in EFGs are listed in Table \ref{table: hyper-para of in EFGs}. Note that, in the context of MMD and Reg-Method, we adhere to their default parameter settings. We only conduct basic grid search slightly based on the default settings for fine-tuning procedures, resulting in performance comparable to or surpassing those reported in~\cite{DBLP:conf/iclr/SokotaDKLLMBK23}.

\begin{table*}[ht]
  
  \centering
    \begin{tabular}{l|l|l|l|l}
    \toprule

    & $\text{MoRM}^{+}$   
    & $\text{MoMWU}$ 
    & OGDA
    & OMWU
    \\
    \midrule
    3*3 matrix
    & $\beta=-0.04, k=10$ 
    & $\eta= 1.0,$\quad $\beta=-0.06, k=50$
    & $\eta= 1.0$
    & $\eta= 1.0$
    \\
    \hline
    Random Game (size=25)
    & $\beta=-0.02, k=70$
    & $\eta= 7.0,$\quad $\beta=-0.02, k=100$ 
    & $\eta= 4.0$
    & $\eta= 5.0$
    \\
    \hline
    Random Game (size=50)
    & $\beta=-0.005, k=70$
    & $\eta= 7.0,$\quad $\beta=-0.02, k=100$
    & $\eta= 5.0$
    & $\eta= 7.0$
    \\
    \hline
    Random Game (size=75)
    & $\beta=-0.003, k=35$
    & $\eta= 9.0,$\quad $\beta=-0.02, k=100$
    & $\eta= 7.0$
    & $\eta= 9.0$
    \\
    \bottomrule
  \end{tabular}
  \caption{Hyper-Parameter Settings in NFGs. }
  \label{table: hyper-para of in NFGs}
\end{table*}

\begin{table*}[ht]
  
  \centering
    \begin{tabular}{l|l|l|l}
    \toprule

    & $\text{MoCFR}^{+}$ 
    & $\text{DMoGDA}$ 
    & OGDA

    \\
    \midrule

     Kuhn Poker
    & $\beta=-0.2$\quad$k=5$
    &$\eta= 2$\quad $\beta=-0.1$\quad $k=10$
    & $\eta= 1.5$
    \\
    \hline
     Goofspiel-4
    & $\beta=-0.02$\quad$k=10$
    & $\eta= 1$\quad $\beta=-0.04$\quad $k=70$
    & $\eta= 0.5$
    \\
    \hline
     Liar's dice-4
    & $\beta=-0.02$\quad$k=40$
    & $\eta= 3$\quad $\beta=-0.005$\quad $k=10$
    & $\eta= 3$
    \\
    \hline
     Leduc Poker
    & $\beta=-0.01$\quad$k=30$
    & $\eta= 4$\quad $\beta=-0.005$\quad $k=100$
    & $\eta= 3$
    \\
    \hline
     Goofspiel-5
    & $\beta=-0.01$\quad$k=50$
    & $\eta= 0.8$\quad $\beta=-0.02$\quad $k=100$
    & $\eta= 0.8$
    \\
    \hline
     Liar's dice-5
    & $\beta=-0.003$\quad$k=100$
    & $\eta= 4$\quad $\beta=-0.0005$\quad $k=50$
    & $\eta= 3$
    \\
    \bottomrule
  \end{tabular}
  \caption{Hyper-Parameter Settings in EFGs. }
  \label{table: hyper-para of in EFGs}
\end{table*}


\subsection{Ablation experiment}\label{sec: Ablation experiment} 
This section will further explore the practical impact of our momentum buffer (related to parameter $k$) on algorithm performance. Specifically, we will apply the DMoGDA algorithm to the Kuhn and Leduc environments, analyze the effects of different momentum parameter $k$ on the algorithm's behavior and convergence speed, and demonstrate its performance in practical applications through numerical experiments.

\begin{figure}[ht]
  \centering
  \subfigure[]{
    \includegraphics[width=0.48\textwidth]{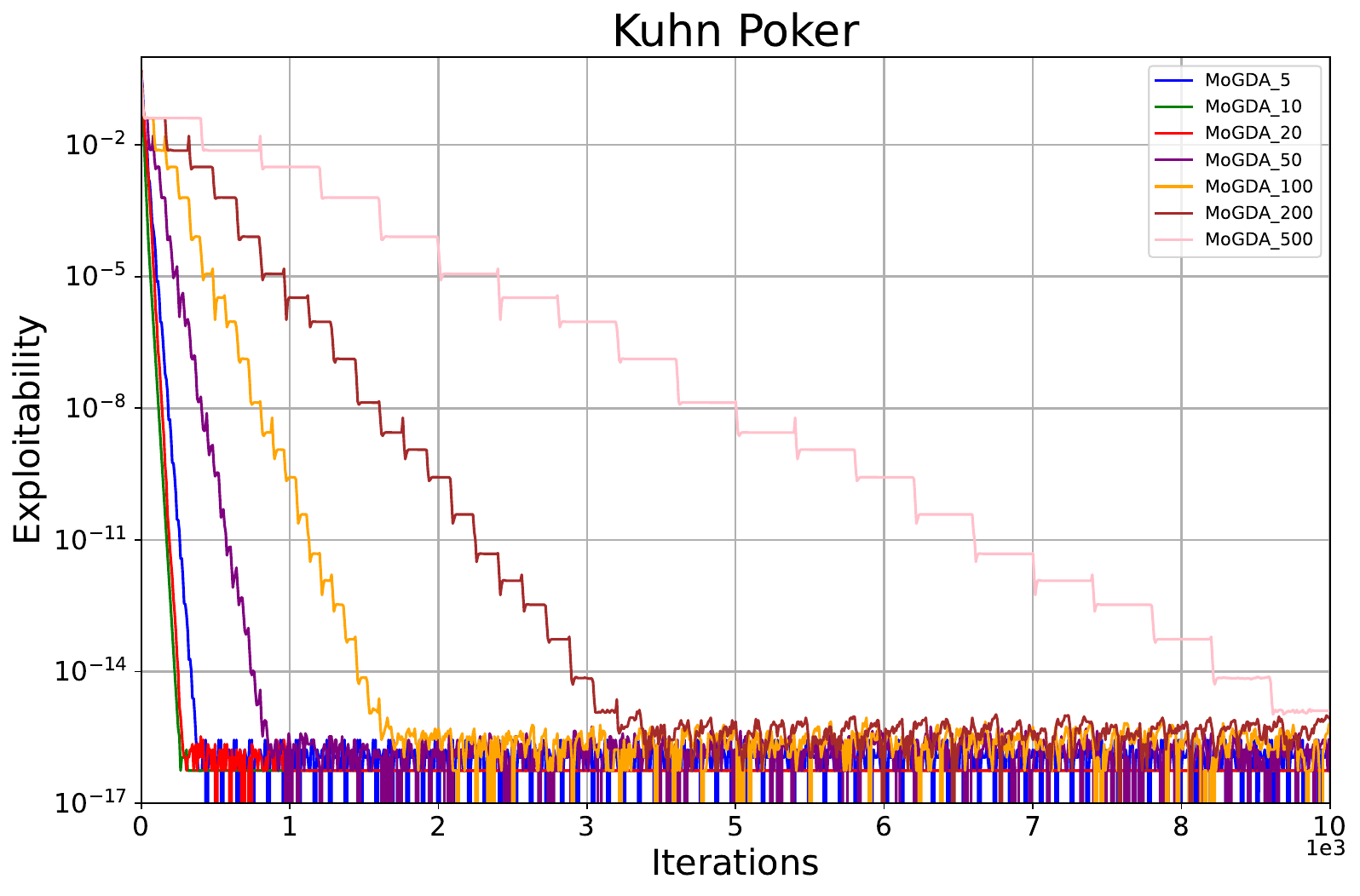}
    \label{curl}
    
  }
  \hfill
  \subfigure[]{
    \includegraphics[width=0.48\textwidth]{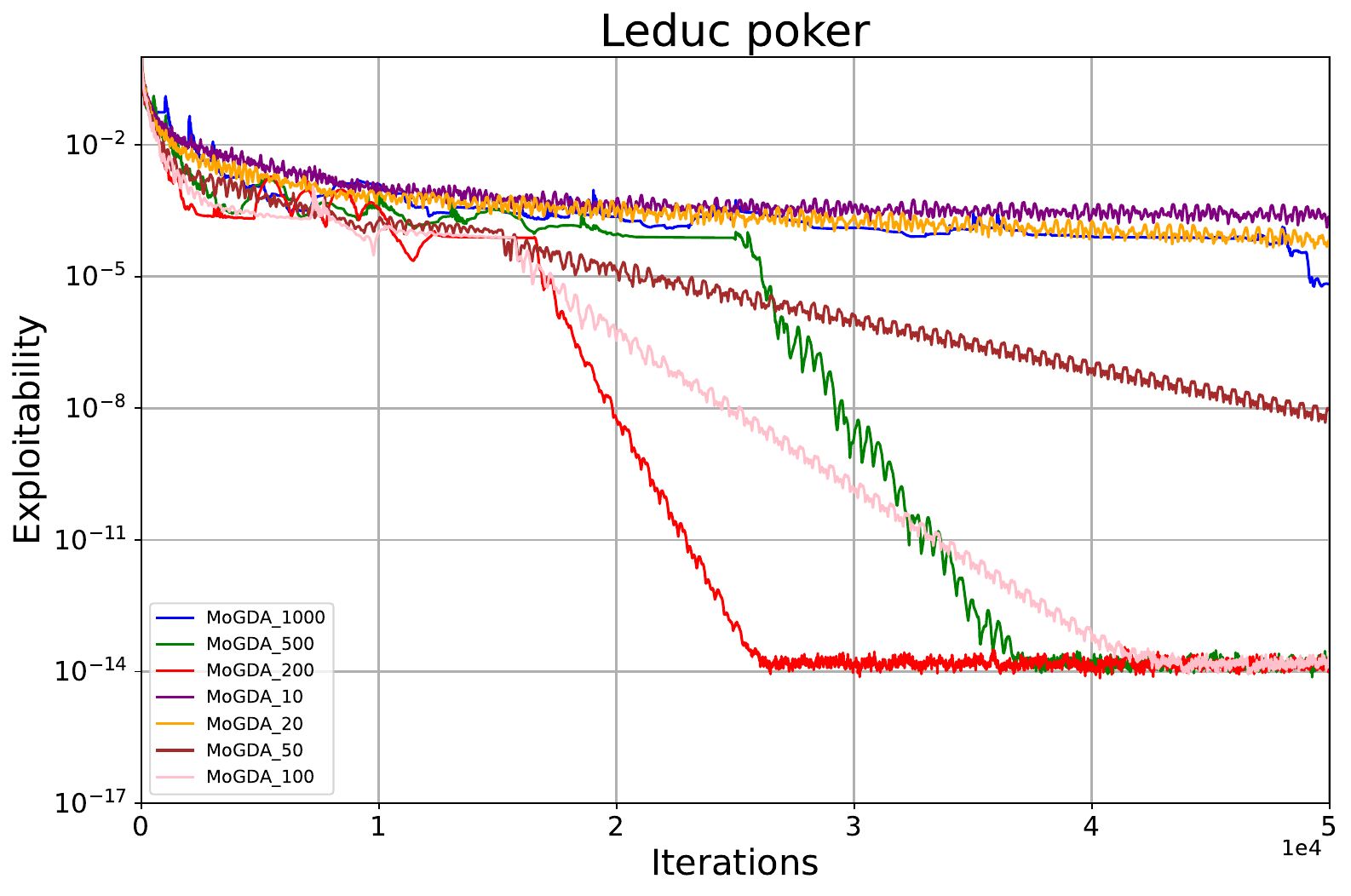}
    \label{Ablation results leduc}
  }
  \caption{The last-iterate convergence results of DMoGDA with different parameters k, in Kuhn Poker (left) and Leduc Poker (right).}
  \label{Ablation results}
\end{figure}


The experiment results above show that when the value of $k$ falls within an appropriate range, the algorithm's performance tends to stabilize. Beyond this point, further increasing $k$ does not significantly improve the algorithm's performance.
Additionally, as shown in Figure \ref{Ablation results}, selecting a large value of $k$ results in the algorithm exhibiting a distinct stair-step pattern. 
Intuitively, this phenomenon is akin to the process of stretching a spring to a certain length, at which point the force balance is achieved with the force introduced by the learning dynamics.
The time (or the optimal $k$) required to reach force balance is relative to the learing rate ($\eta$) and the coefficient($\beta $). This means that we do not need to use all historical momentum information; instead, by observing the stair-like curve that emerges when the "force balance" state occurs, we can determine the optimal value (or even an earlier value) for switching attachments, enabling the algorithm to converge more efficiently.Thus, in practice, we use a finite $k$ to focus the algorithm on recent momentum information.

\subsection{Theoretical Proofs}\label{appendix proof}

\subsubsection{Proof of Proposition~\ref{prop: unconstrained NM}}


\begin{proof}[Proof of Proposition \ref{prop: unconstrained NM}]

    To discretize $\dot{Z} = -F(Z)$, we have:
    \begin{align*}
        &\text{Euler Implicit Discretization: }
        \boldsymbol{z}_{t+1} - \boldsymbol{z}_{t} = -\delta F(\boldsymbol{z}_{t+1}),
        \\
        &\text{Euler Explicit Discretization: }
        \boldsymbol{z}_{t+1} - \boldsymbol{z}_{t} = -\delta F(\boldsymbol{z}_{t}).
    \end{align*}
    Therefore, Equation (\ref{eq: continuous negative}) can be discretized as follows:
    \begin{align}
        &\boldsymbol{z}_{t+1} - \boldsymbol{z}_{t} = \delta g(\boldsymbol{z}_{t+1}),\label{eq: prop 1}
        \\
        & g(\boldsymbol{z}_{t+1}) - g(\boldsymbol{z}_{t}) = -\delta F(\boldsymbol{z}_{t}) - \mu\delta g(\boldsymbol{z}_{t}).\label{eq: prop 2}
    \end{align}
    By combining (\ref{eq: prop 1}) and (\ref{eq: prop 2}), we get
    \begin{align*}
        \boldsymbol{z}_{t+1} - \boldsymbol{z}_{t} = 
        -\eta F(\boldsymbol{z}_{t}) + (1-\mu\delta)(\boldsymbol{z}_{t} - \boldsymbol{z}_{t-1}). 
    \end{align*}
\end{proof}

\subsubsection{Proof of Theorem~\ref{thrm: convergence of MoMWU}}


\begin{proof}[Proof of Theorem~\ref{thrm: convergence of MoMWU}]
    Under the setup in Theorem~\ref{thrm: convergence of MoMWU}, we get
    \begin{align}
        \boldsymbol{z}_{t+1} \propto \exp(-\eta \boldsymbol{L}_{t})
        \text{ and }
        \boldsymbol{z}_{att} \propto \exp(-\eta \boldsymbol{L}_{att}).
    \end{align}
    Thus, from the definition of $\boldsymbol{L}_{t}$, we have
    \begin{align}
        \boldsymbol z_{t+1} \propto
        \boldsymbol z_{t}^{1+\beta}\exp(-\eta F(\boldsymbol z_t) + \eta\beta\boldsymbol L_\text{ref})
        \propto
        \boldsymbol z_{t}^{1+\beta}\boldsymbol z_{att}^{-\beta}\exp(-\eta F(\boldsymbol z_t) ).
    \end{align}
    We invoke the following lemma first.
    \begin{lemma}
        Let $\mathcal{Z}$ be a convex set that satisfies $A\boldsymbol{z}=b$ for all $\boldsymbol{z}\in\mathcal{Z}$ for a matrix $A$ and a vector $b$. Then for $\boldsymbol g$ and $\boldsymbol g'$, and any $\boldsymbol{z}_*$, we have
        \begin{align*}
            D_\psi(\boldsymbol z_*, T(\boldsymbol{g}')) - D_\psi(\boldsymbol z_*,T(\boldsymbol g))
            +D_\psi(T(\boldsymbol g'),T(\boldsymbol g)) = \langle{\boldsymbol g - \boldsymbol g',  \boldsymbol{z}_*-T(\boldsymbol{g}')}\rangle.
        \end{align*}
    \end{lemma}
    \begin{proof}
        By the first-order optimality condition of $T(\boldsymbol{g})$ and $T(\boldsymbol{g}')$, there exists $\boldsymbol{v},\boldsymbol{v}'$ such that
        \begin{align*}
            &\boldsymbol{g} + \nabla \psi(T(\boldsymbol{g})) = A^\top\boldsymbol{v},
            \\
            &\boldsymbol{g}' + \nabla \psi(T(\boldsymbol{g}')) = A^\top\boldsymbol{v}'.
        \end{align*}
        Thus, we have
        \begin{align*}
            \langle{\boldsymbol{g}, \boldsymbol{z}_*-T(\boldsymbol{g})}\rangle
            &=
            \langle{\nabla \psi(T(\boldsymbol{g}))+A^\top\boldsymbol{v}, \boldsymbol{z}_*-T(\boldsymbol{g})}\rangle
            \\&=
            \langle{\nabla \psi(T(\boldsymbol{g})), \boldsymbol{z}_*-T(\boldsymbol{g})}\rangle + \langle{A^\top\boldsymbol{v}, \boldsymbol{z}_*-T(\boldsymbol{g})}\rangle
            \\&=
            \langle{\nabla \psi(T(\boldsymbol{g})), \boldsymbol{z}_*-T(\boldsymbol{g})}\rangle + \langle{\boldsymbol{v}, b-b}\rangle
            \\&=
            \langle{\nabla \psi(T(\boldsymbol{g})), \boldsymbol{z}_*-T(\boldsymbol{g})}\rangle.
        \end{align*}
        Similarly, we have $\langle{\boldsymbol{g}', \boldsymbol{z}_*-T(\boldsymbol{g}')}\rangle = \langle{\nabla \psi(T(\boldsymbol{g}')), \boldsymbol{z}_*-T(\boldsymbol{g}')}\rangle$.
        Thus, since $D_\psi(\boldsymbol{z}', \boldsymbol{z}) = \psi(\boldsymbol{z}') - \psi(\boldsymbol{z}) - \langle{\nabla \psi(\boldsymbol{z}), \boldsymbol{z}'-\boldsymbol{z}}\rangle$, we have $D_\psi(\boldsymbol{z}_*, T(\boldsymbol{g})) = \psi(\boldsymbol{z}_*) - \psi(T(\boldsymbol{g})) - \langle{\boldsymbol{g}, \boldsymbol{z}_*-T(\boldsymbol{g})}\rangle$, $D_\psi(\boldsymbol{z}_*, T(\boldsymbol{g}')) = \psi(\boldsymbol{z}_*) - \psi(T(\boldsymbol{g}')) - \langle{\boldsymbol{g}', \boldsymbol{z}_*-T(\boldsymbol{g}')}\rangle$ and $D_\psi(T(\boldsymbol{g}'), T(\boldsymbol{g})) = \psi(T(\boldsymbol{g}') - \psi(T(\boldsymbol{g})) - \langle{\boldsymbol{g}, T(\boldsymbol{g}')-T(\boldsymbol{g})}\rangle$.

        Therefore, we get 
        \begin{align*}
            &D_\psi(\boldsymbol z_*, T(\boldsymbol{g}')) - D_\psi(\boldsymbol z_*,T(\boldsymbol g))
            +D_\psi(T(\boldsymbol g'),T(\boldsymbol g))
            \\=& 
            \psi(\boldsymbol{z}_*) - \psi(T(\boldsymbol{g}')) - \langle{\boldsymbol{g}', \boldsymbol{z}_*-T(\boldsymbol{g}')}\rangle 
            -\psi(\boldsymbol{z}_*) + \psi(T(\boldsymbol{g})) + \langle{\boldsymbol{g}, \boldsymbol{z}_*-T(\boldsymbol{g})}\rangle
            +\psi(T(\boldsymbol{g}') - \psi(T(\boldsymbol{g})) - \langle{\boldsymbol{g}, T(\boldsymbol{g}')-T(\boldsymbol{g})}\rangle
            \\=& 
            - \langle{\boldsymbol{g}', \boldsymbol{z}_*-T(\boldsymbol{g}')}\rangle 
            + \langle{\boldsymbol{g}, \boldsymbol{z}_*-T(\boldsymbol{g})}\rangle
            - \langle{\boldsymbol{g}, T(\boldsymbol{g}')-T(\boldsymbol{g})}\rangle
            \\=&
            \langle{\boldsymbol{g}-\boldsymbol{g}',  \boldsymbol{z}_*-T(\boldsymbol{g}')}\rangle.
        \end{align*}
    \end{proof}

    Thus, from the above lemma, we have
    \begin{align}\label{eq: thrm4.2 eq1}
        &D_\psi(\boldsymbol z_*,\boldsymbol z_{t+1}) - D_\psi(\boldsymbol z_*,\boldsymbol z_{t}) + D_\psi(\boldsymbol z_{t+1},\boldsymbol z_{t}) = -\eta\langle{\boldsymbol L_{t} - \boldsymbol L_{t-1}, \boldsymbol z_{t+1} - \boldsymbol z_{*}}\rangle
        \nonumber\\
        =&
        -\eta\langle{F(\boldsymbol z_t) + \beta\big( \boldsymbol L_{t-1} - \boldsymbol L_{att} \big), \boldsymbol z_{t+1} - \boldsymbol z_{*}}\rangle
        \nonumber\\
        =&
        \langle{-\eta F(\boldsymbol z_t) + \beta\log\frac{\boldsymbol z_{t}}{\boldsymbol z_{att}}, \boldsymbol z_{t+1} - \boldsymbol z_{*}}\rangle
        \nonumber\\
        =&
        -\eta\langle{ F(\boldsymbol z_t), \boldsymbol z_{t+1} - \boldsymbol z_{*}}\rangle
          - \beta\langle{\log\frac{\boldsymbol z_{t}}{\boldsymbol z_{att}}, \boldsymbol{z}_{t}-\boldsymbol z_{t+1} }\rangle
        +\beta\langle{\log\frac{\boldsymbol z_{t}}{\boldsymbol z_{att}}, \boldsymbol{z}_{t}-\boldsymbol z_{*} }\rangle.
    \end{align}
    The second term of (\ref{eq: thrm4.2 eq1}) can be rewritten as
    \begin{equation}
        \begin{aligned}
        &\langle{\log\frac{\boldsymbol z_{t}}{\boldsymbol z_{att}}, \boldsymbol{z}_{t}-\boldsymbol z_{t+1} }\rangle =
        D_\text{KL}(\boldsymbol z_{t}, \boldsymbol z_{att}) - D_\text{KL}(\boldsymbol z_{t+1}, \boldsymbol z_{att}) + D_\text{KL}(\boldsymbol z_{t+1}, \boldsymbol z_{t})
        \\&=
        D_\text{KL}(\boldsymbol z_{t}, \boldsymbol z_{att}) - D_\text{KL}(\boldsymbol z_{*}, \boldsymbol z_{att}) - D_\text{KL}(\boldsymbol z_{t+1}, \boldsymbol z_{*}) + \langle{\log\frac{\boldsymbol z_{*}}{\boldsymbol z_{att}}, \boldsymbol z_{*}-\boldsymbol z_{t+1}}\rangle + D_\text{KL}(\boldsymbol z_{t+1}, \boldsymbol z_{t}),
        \end{aligned}
    \end{equation}
    and the third term of (\ref{eq: thrm4.2 eq1}) can be rewritten as
    \begin{equation}
        \begin{aligned}
        &\langle{\log\frac{\boldsymbol z_{t}}{\boldsymbol z_{att}}, \boldsymbol{z}_{t}-\boldsymbol z_{*} }\rangle =
        D_\text{KL}(\boldsymbol z_{t}, \boldsymbol z_{att}) - D_\text{KL}(\boldsymbol z_{*}, \boldsymbol z_{att}) + D_\text{KL}(\boldsymbol z_{*}, \boldsymbol z_{t}).
        \end{aligned}
    \end{equation}
    Therefore, we have
    \begin{align}
        &D_\psi(\boldsymbol z_*,\boldsymbol z_{t+1}) - D_\psi(\boldsymbol z_*,\boldsymbol z_{t}) + D_\psi(\boldsymbol z_{t+1},\boldsymbol z_{t}) 
        \nonumber\\
        =&
        -\eta\langle{ F(\boldsymbol z_t), \boldsymbol z_{t+1} - \boldsymbol z_{*}}\rangle
          - \beta\langle{\log\frac{\boldsymbol z_{t}}{\boldsymbol z_{att}}, \boldsymbol{z}_{t}-\boldsymbol z_{t+1} }\rangle
        +\beta\langle{\log\frac{\boldsymbol z_{t}}{\boldsymbol z_{att}}, \boldsymbol{z}_{t}-\boldsymbol z_{*} }\rangle
        \nonumber\\
        \le&
        -\eta\langle{ F(\boldsymbol z_t), \boldsymbol z_{t+1} - \boldsymbol z_{*}}\rangle
        \nonumber\\&-
        \beta D_\text{KL}(\boldsymbol z_{t}, \boldsymbol z_{att}) + \beta D_\text{KL}(\boldsymbol z_{*}, \boldsymbol z_{att}) - \beta\langle{\log\frac{\boldsymbol z_{*}}{\boldsymbol z_{att}}, \boldsymbol z_{*}-\boldsymbol z_{t+1}}\rangle - \beta D_\text{KL}(\boldsymbol z_{t+1}, \boldsymbol z_{t})
        \nonumber\\&+
        \beta D_\text{KL}(\boldsymbol z_{t}, \boldsymbol z_{att}) -\beta D_\text{KL}(\boldsymbol z_{*}, \boldsymbol z_{att}) +\beta D_\text{KL}(\boldsymbol z_{*}, \boldsymbol z_{t})
        \nonumber\\
        =&
        -\eta\langle{ F(\boldsymbol z_t), \boldsymbol z_{t+1} - \boldsymbol z_{*}}\rangle
        -\beta\langle{\log\frac{\boldsymbol z_{*}}{\boldsymbol z_{att}}, \boldsymbol z_{*}-\boldsymbol z_{t+1}}\rangle - \beta D_\text{KL}(\boldsymbol z_{t+1}, \boldsymbol z_{t})
        +\beta D_\text{KL}(\boldsymbol z_{*}, \boldsymbol z_{t}),
    \end{align}
    which implies that
    \begin{align}
        &D_\psi(\boldsymbol z_*,\boldsymbol z_{t+1}) - (1+\beta)D_\psi(\boldsymbol z_*,\boldsymbol z_{t}) + (1+\beta)D_\psi(\boldsymbol z_{t+1},\boldsymbol z_{t}) 
        \nonumber\\
        =&
        -\eta\langle{ F(\boldsymbol z_t), \boldsymbol z_{t+1} - \boldsymbol z_{*}}\rangle
        -\beta\langle{\log\frac{\boldsymbol z_{*}}{\boldsymbol z_{att}}, \boldsymbol z_{*}-\boldsymbol z_{t+1}}\rangle 
        \nonumber\\
        =&
        -\eta\langle{ F(\boldsymbol z_{t+1}), \boldsymbol z_{t+1} - \boldsymbol z_{*}}\rangle
        -\beta\langle{\log\frac{\boldsymbol z_{*}}{\boldsymbol z_{att}}, \boldsymbol z_{*}-\boldsymbol z_{t+1}}\rangle 
        \nonumber\\&-\eta\langle{ F(\boldsymbol z_{t})-F(\boldsymbol z_{t+1}), \boldsymbol z_{t+1} - \boldsymbol z_{*}}\rangle
        \nonumber\\\le&
        -\eta\langle{ F(\boldsymbol z_{*}), \boldsymbol z_{t+1} - \boldsymbol z_{*}}\rangle
        -\beta\langle{\log\frac{\boldsymbol z_{*}}{\boldsymbol z_{att}}, \boldsymbol z_{*}-\boldsymbol z_{t+1}}\rangle 
        \nonumber\\&-\eta\langle{ F(\boldsymbol z_{t})-F(\boldsymbol z_{t+1}), \boldsymbol z_{t+1} - \boldsymbol z_{*}}\rangle.
    \end{align}
    Assume that $\boldsymbol z_*$ is the unique regularized equilibrium of $\min_{\boldsymbol x}\max_{\boldsymbol y} \boldsymbol x^\mathsf{T}\boldsymbol{G}\boldsymbol y + \frac{\alpha}{\eta}D_\text{KL}(\boldsymbol x,\boldsymbol x_{att}) - \frac{\alpha}{\eta}D_\text{KL}(\boldsymbol y,\boldsymbol y_{att})$. Then, we have
    \begin{align}
        &D_\psi(\boldsymbol z_*,\boldsymbol z_{t+1}) - (1+\beta)D_\psi(\boldsymbol z_*,\boldsymbol z_{t}) + (1+\beta)D_\psi(\boldsymbol z_{t+1},\boldsymbol z_{t}) 
        \nonumber\\
        \le&
        -\eta\langle{ F(\boldsymbol z_{*}), \boldsymbol z_{t+1} - \boldsymbol z_{*}}\rangle
        -\beta\langle{\log\frac{\boldsymbol z_{*}}{\boldsymbol z_{att}}, \boldsymbol z_{*}-\boldsymbol z_{t+1}}\rangle 
        \nonumber\\&-\eta\langle{ F(\boldsymbol z_{t})-F(\boldsymbol z_{t+1}), \boldsymbol z_{t+1} - \boldsymbol z_{*}}\rangle
        \nonumber\\
        \le&
        -\eta\langle{ F(\boldsymbol z_{t})-F(\boldsymbol z_{t+1}), \boldsymbol z_{t+1} - \boldsymbol z_{*}}\rangle
        \nonumber
        \\=&
        -\eta\langle{ F(\boldsymbol z_{t})-F(\boldsymbol z_{t+1}), \boldsymbol z_{t} - \boldsymbol z_{*}}\rangle
        -\eta\langle{ F(\boldsymbol z_{t})-F(\boldsymbol z_{t+1}), \boldsymbol z_{t+1} - \boldsymbol z_{t}}\rangle,
    \end{align}
    where the second inequality follows from the definition of $\boldsymbol{z}_*$. 

    Therefore, we have
    \begin{align}
        &D_\psi(\boldsymbol z_*,\boldsymbol z_{t+1}) - (1+\beta)D_\psi(\boldsymbol z_*,\boldsymbol z_{t}) + (1+\beta)D_\psi(\boldsymbol z_{t+1},\boldsymbol z_{t}) 
        \nonumber\\
        \le&
        -\eta\langle{ F(\boldsymbol z_{t})-F(\boldsymbol z_{t+1}), \boldsymbol z_{t} - \boldsymbol z_{*}}\rangle
        -\eta\langle{ F(\boldsymbol z_{t})-F(\boldsymbol z_{t+1}), \boldsymbol z_{t+1} - \boldsymbol z_{t}}\rangle
        \nonumber\\\le&
        \eta\rho\Vert{F(\boldsymbol z_{t})-F(\boldsymbol z_{t+1})}\Vert^2_1 + \frac{\eta}{2\rho}\Vert{\boldsymbol z_{t} - \boldsymbol z_{*}}\Vert^2_1+ \frac{\eta}{2\rho}\Vert{\boldsymbol z_{t+1} - \boldsymbol z_{t}}\Vert^2_1
        \nonumber\\\le&
        \frac{\eta}{2\rho}\Vert{\boldsymbol z_{t} - \boldsymbol z_{*}}\Vert^2_1+ (\frac{\eta}{2\rho}+\eta\rho)\Vert{\boldsymbol z_{t+1} - \boldsymbol z_{t}}\Vert^2_1
        \nonumber\\\le&
        \frac{\eta}{\rho}D_\text{KL}(\boldsymbol z_{*},\boldsymbol z_{t}) + (\frac{\eta}{\rho}+2\eta\rho)D_\text{KL}(\boldsymbol z_{t+1} , \boldsymbol z_{t}),
    \end{align}
    where the second inequality follows from the Yong's inequality and the last one follows from the Pinsker inequality.

    By setting $\rho=\frac{2\eta}{-\beta}$, we get
    \begin{align*}
        &D_\psi(\boldsymbol z_*,\boldsymbol z_{t+1}) 
        \\\le&
        (1+\frac{\beta}{2})D_\text{KL}(\boldsymbol z_{*},\boldsymbol z_{t}) + (-1-\frac{3}{2}\beta-\frac{4}{\beta}\eta^2)D_\text{KL}(\boldsymbol z_{t+1} , \boldsymbol z_{t})
    \end{align*}
    By setting $-\frac{2}{3}<\beta<0$ and $\eta\le\frac{\sqrt{-(1+\frac{3}{2}\beta)\beta}}{2}$, we can obtain our desired results:
    \begin{align*}
        &D_\psi(\boldsymbol z_*,\boldsymbol z_{t+1}) 
        \le
        (1+\frac{\beta}{2})D_\text{KL}(\boldsymbol z_{*},\boldsymbol z_{t}) 
        \\\le&
        (1+\frac{\beta}{2})^{t+1}D_\text{KL}(\boldsymbol z_{*},\boldsymbol z_{0}) .
    \end{align*}
    Thus, the proof is completed.
\end{proof}

\subsubsection{Proof of Theorem~\ref{thrm: DualityGap of MoMWU}}

\begin{proof}[Proof of Theorem~\ref{thrm: DualityGap of MoMWU}]
    From the definition of duality gap, we have
    \begin{align*}
        {DualityGap}(\boldsymbol z_t)=&
        \sup_{\boldsymbol{z}}\Big\{\langle{F(\boldsymbol{z}_{*}), \boldsymbol{z}_{*} - \boldsymbol{z}}\rangle
        -\langle{F(\boldsymbol{z}_{*}), \boldsymbol{z}_{*} - \boldsymbol{z}_{t}}\rangle
        \\&+\langle{F(\boldsymbol{z}_{t})-F(\boldsymbol{z}_{*}), \boldsymbol{z}_{t} - \boldsymbol{z}}\rangle
        \Big\}
        \\\le&
        DualityGap(\boldsymbol{z}_{*})
        -\langle{F(\boldsymbol{z}_{*}), \boldsymbol{z}_{*} - \boldsymbol{z}_{t}}\rangle
        \\&+
        \sup_{\boldsymbol{z}}\Big\{\langle{F(\boldsymbol{z}_{t})-F(\boldsymbol{z}_{*}), \boldsymbol{z}_{t} - \boldsymbol{z}}\rangle
        \Big\}
        \\\le&
        DualityGap(\boldsymbol{z}_{*})
        +\Vert{F(\boldsymbol{z}_{*})}\Vert\Vert{\boldsymbol{z}_{*} - \boldsymbol{z}_{t}}\Vert
        \\&+
        \Vert{F(\boldsymbol{z}_{t})-F(\boldsymbol{z}_{*})}\Vert \text{diam}(\boldsymbol{Z})
        \\\le&
        DualityGap(\boldsymbol{z}_{*})
        +(\Vert{F(\boldsymbol{z}_{*})}\Vert+\lambda\cdot\text{diam}(\boldsymbol{Z}))\Vert{\boldsymbol{z}_{*} - \boldsymbol{z}_{t}}\Vert_1
        \\\le&
        DualityGap(\boldsymbol{z}_{*})
        +(\Vert{F(\boldsymbol{z}_{*})}\Vert+\lambda\cdot\text{diam}(\boldsymbol{Z}))
        \sqrt{2D_\text{KL}(\boldsymbol{z}_{*},\boldsymbol{z}_{t})}
        \\\le&
        DualityGap(\boldsymbol{z}_{*})
        \\&+(\Vert{F(\boldsymbol{z}_{*})}\Vert+\lambda\cdot\text{diam}(\boldsymbol{Z}))
        \sqrt{2D_\text{KL}(\boldsymbol{z}_{*},\boldsymbol{z}_{0})}(1+\frac{\beta}{2})^{\frac{t}{2}},
    \end{align*}
    where the last inequality follows from Theorem~\ref{thrm: convergence of MoMWU}. 

    Furthermore, $DualityGap(\boldsymbol{z}_{*})$ can be upper bounded as:
    \begin{align*}
        DualityGap(\boldsymbol{z}_{*})\le&
        \sup_{\boldsymbol{z}}\Big\{
        \langle{F(\boldsymbol{z}_{*}) + \frac{\beta}{\eta}\log\frac{\boldsymbol{z}_{*}}{\boldsymbol{z}_{att}},\boldsymbol z_*-\boldsymbol z}\rangle
        -\frac{\beta}{\eta}\langle{\log\frac{\boldsymbol{z}_{*}}{\boldsymbol{z}_{att}},\boldsymbol z_*-\boldsymbol z}\rangle
        \Big\}
        \\\le&
        \sup_{\boldsymbol{z}}\Big\{
        -\frac{\beta}{\eta}\langle{\log\frac{\boldsymbol{z}_{*}}{\boldsymbol{z}_{att}},\boldsymbol z_*-\boldsymbol z}\rangle
        \Big\}
        \\\le&
        -\frac{\beta}{\eta}\text{diam}(\boldsymbol Z)
        \cdot\Vert{\log\frac{\boldsymbol{z}_{*}}{\boldsymbol{z}_{att}}}\Vert,
    \end{align*}
    where the second inequality follows from the definition of $\boldsymbol{z}_*$ and the last one follows from the Cauchy-Schwarz inequality. 

    Thus, the proof is completed. 
    
\end{proof}

\subsubsection{Proof of Theorem \ref{thrm: convergence to NE}}

\begin{proof}[Proof of Theorem \ref{thrm: convergence to NE}]

We re-denote the set of Nash equilibria of the original game as $\mathcal{Z}_{*} = \{\boldsymbol{z}_*\}$, and the regularized equilibrium under $\boldsymbol{z}_{att}$ as $\boldsymbol{z}_{att}^*$, and we use the superscript $(n)$ to denote the n-th updated $\boldsymbol{z}_{att}$ as $\boldsymbol{z}_{att}^{(n)}$ (similar to $\boldsymbol{z}_{att}^{(n),*}$). 

We begin with following useful lemmas. 
\begin{lemma}[Adapted from Lemma 6.2 in \citet{DBLP:journals/corr/abs-2208-09855}]\label{lemma: exact 1}
    For any $n\ge0$, if $\boldsymbol{z}_{att}^{(n)}\in \mathcal{Z}\backslash\mathcal{Z}_*$, we have $\min_{\boldsymbol{z}_{*}}D_\psi(\boldsymbol{z}_*, \boldsymbol{z}_{att}^{(n+1)}) < \min_{\boldsymbol{z}_{*}}D_\phi(\boldsymbol{z}_*, \boldsymbol{z}_{att}^{(n)})$. 
    Otherwise, $\boldsymbol{z}_{att}^{(n+1)}= \boldsymbol{z}_{att}^{(n)}\in\mathcal{Z}_*$. 
\end{lemma}

\begin{lemma}[Adapted from Lemma 6.3 in \citet{DBLP:journals/corr/abs-2208-09855}]\label{lemma: exact 2}
    Let $f(\boldsymbol{z}_{att}) = \boldsymbol{z}_{att}^{*}$ be a map that maps the strategy $\boldsymbol{z}_{att}$ to its corresponding regularized equilibrium $\boldsymbol{z}_{att}^{*}$. Then, $f$ is continuous. 
\end{lemma}

Denote $b = \lim_{n\rightarrow\infty} \min_{\boldsymbol{z}_{*}\in\mathcal{Z}_{*}} D_{\psi}(\boldsymbol{z}_{*}, \boldsymbol{z}_{att}^{(n)}) \ge0$. We next prove that $b=0$ and thus $\boldsymbol{z}_{att}^{(n)}$ converges to the set of Nash equilibria. 

By contradiction, we suppose that $b>0$ and define $B = \min_{\boldsymbol{z}_{*}\in\mathcal{Z}_{*}} D_{\psi}(\boldsymbol{z}_{*}, \boldsymbol{z}_{att}^{(0)})$. 
From Lemma~\ref{lemma: exact 1}, $\min_{\boldsymbol{z}_{*}}D_\psi(\boldsymbol{z}_*, \boldsymbol{z}_{att}^{(n)})$ monotonically decreases, and thus each $\boldsymbol{z}_{att}^{(n)})$ falls into the set $\Omega_{b,B} = \{ \boldsymbol{z}_{att}: b \le \min_{\boldsymbol{z}_{*}}D_\psi(\boldsymbol{z}_*, \boldsymbol{z}_{att}) \le B \}$. 
From Lemma~\ref{lemma: exact 2}, $\min_{\boldsymbol{z}_{*}}D_\psi(\boldsymbol{z}_*, \boldsymbol{z}_{att})$ is continuous, and thus $\Omega_{b,B}$ is a compact set due to the boundedness of $\mathcal{Z}$. 

From lemma~\ref{lemma: exact 2}, $\Delta V(\boldsymbol{z}_{att}) := \min_{\boldsymbol{z}_{*}}D_\psi(\boldsymbol{z}_*, f(\boldsymbol{z}_{att})) - \min_{\boldsymbol{z}_{*}}D_\psi(\boldsymbol{z}_*, \boldsymbol{z}_{att})$ is also continuous. 
Thus $\Delta V(\boldsymbol{z}_{att})$ has a maximum over a compact set, i.e., $M=\max_{\boldsymbol{z}_{att}\in\Omega_{b,B}} \Delta V(\boldsymbol{z}_{att})$ exists. 
From Lemma~\ref{lemma: exact 1}, $M<0$, and thus we have:
\begin{align*}
    \min_{\boldsymbol{z}_{*}}D_\psi(\boldsymbol{z}_*, \boldsymbol{z}_{att}^{(n)}) &= 
    \min_{\boldsymbol{z}_{*}}D_\psi(\boldsymbol{z}_*, \boldsymbol{z}_{att}^{(0)}) + \sum_{l=0}^{n-1}\Big( \min_{\pi_{*}\in\Pi_{*}} \min_{\boldsymbol{z}_{*}}D_\psi(\boldsymbol{z}_*, \boldsymbol{z}_{att}^{(l+1)}) - \min_{\boldsymbol{z}_{*}}D_\psi(\boldsymbol{z}_*, \boldsymbol{z}_{att}^{(l)}) \Big)
    \\&\le
    B + nM.
\end{align*}
This implies that $\min_{\boldsymbol{z}_{*}}D_\psi(\boldsymbol{z}_*, \boldsymbol{z}_{att}^{(n)}) < 0$ for $n>\frac{-B}{M}$, which is a contradiction since $\min_{\boldsymbol{z}_{*}}D_\psi(\boldsymbol{z}_*, \boldsymbol{z}_{att}^{(n)})\ge0$. 

Thus, the proof is completed. 
\end{proof}

\end{document}